\documentclass[11pt]{article}

\usepackage{fullpage,times,url}

\usepackage{bbm,amsthm,amsfonts,amsmath,amssymb,epsfig,color,float,graphicx,verbatim}
\usepackage{epstopdf}

\usepackage{algorithm,algorithmic,natbib}
\usepackage{mathtools}
\usepackage{multicol}

\usepackage{hyperref}
\hypersetup{
	colorlinks   = true, 
	urlcolor     = blue, 
	linkcolor    = blue, 
	citecolor   = black 
}

\newtheorem{theorem}{Theorem}
\newtheorem{claim}{Claim}

\newtheorem{lemma}{Lemma}

\newcommand{\relu}[1]{\left[ #1 \right]_+}
\newcommand{\ceil}[1]{\left\lceil #1 \right\rceil}

\newcommand{\set}[1]{\left\lbrace#1\right\rbrace}
\newcommand{\p}[1]{\left( #1 \right)}
\newcommand{\pcc}[1]{\left[ #1 \right]}

\newcommand{\pco}[1]{\left[ #1 \right)}

\newcommand{\reals}{\mathbb{R}}
\newcommand{\E}{\mathbb{E}}

\newcommand{\abs}[1]{\left| #1 \right|}
\newcommand{\poly}{\text{poly}}

\newcommand{\ba}{\mathbf{a}}

\newcommand{\bx}{\mathbf{x}}
\newcommand{\bw}{\mathbf{w}}

\newcommand{\bb}{\mathbf{b}}

\newcommand{\bz}{\mathbf{z}}

\newcommand{\by}{\mathbf{y}}

\newcommand{\bp}{\mathbf{p}}
\newcommand{\bzero}{\mathbf{0}}

\newcommand{\balpha}{\boldsymbol{\alpha}}

\newcommand{\Ocal}{\mathcal{O}}

\newcommand{\Fcal}{\mathcal{F}}

\newcommand{\Ncal}{\mathcal{N}}

\newcommand{\norm}[1]{\left|\left|#1\right|\right|}
\newcommand{\inner}[1]{\langle#1\rangle}

\newcommand{\pr}{\mathbb{P}}

\newcommand\pFq[2]{{}_{#1}F_{#2}}

\newtheorem{assumption}{Assumption}

\newcommand{\secref}[1]{Sec.~\ref{#1}}

\newcommand{\figref}[1]{Fig.~\ref{#1}}
\renewcommand{\eqref}[1]{Eq.~(\ref{#1})}
\newcommand{\lemref}[1]{Lemma~\ref{#1}}

\newcommand{\thmref}[1]{Thm.~\ref{#1}}

\newcommand{\unitball}{B_d}

\title{Depth Separations in Neural Networks:\\What is Actually Being 
Separated?\footnote{Accepted for presentation at the Conference on Learning Theory (COLT) 2019}}
\author{Itay Safran\qquad Ronen Eldan\qquad Ohad Shamir\\
Weizmann Institute of 
Science\\\texttt{\{itay.safran,ronen.eldan,ohad.shamir\}.weizmann.ac.il}}

\date{}

\begin{document}

	\maketitle
	\begin{abstract}
		Existing depth separation results for constant-depth networks 
		essentially show that certain radial functions in $\mathbb{R}^d$, which can be easily approximated with depth $3$ networks, cannot be approximated by depth $2$ networks, even up to constant accuracy, unless their size is exponential in $d$. However, the functions used to demonstrate this are 
		rapidly oscillating, with a Lipschitz parameter scaling polynomially with the dimension $d$ (or equivalently, by scaling the function, the hardness result applies to $\mathcal{O}(1)$-Lipschitz functions only when the target accuracy $\epsilon$ is at most $\text{poly}(1/d)$). In this paper, we 
		study whether such depth separations might still hold in the natural setting of $\mathcal{O}(1)$-Lipschitz radial functions, when $\epsilon$ does not scale with $d$. Perhaps surprisingly, we show that the answer is negative: In contrast to the intuition suggested by previous work, it \emph{is} possible to approximate $\mathcal{O}(1)$-Lipschitz radial functions with depth $2$, size $\text{poly}(d)$ networks, for every constant $\epsilon$. We complement it by showing that approximating such functions is also possible with depth $2$, size $\text{poly}(1/\epsilon)$ networks, for every constant $d$. Finally, we show that it is not possible to have polynomial dependence in both $d,1/\epsilon$ simultaneously. 
		Overall, our results indicate that in order to show depth separations for expressing $\mathcal{O}(1)$-Lipschitz functions with constant accuracy -- if at all possible -- one would need fundamentally different techniques than existing ones in the literature.
	\end{abstract}
	
	\section{Introduction}
	In the past few years, quite a few theoretical works have explored the 
	beneficial effect of depth on increasing the expressiveness of neural 
	networks (e.g., 
	\citet{delalleau2011shallow,martens2013representational,martens2014expressive,montufar2014number,cohen2015expressive,
	telgarsky2016benefits,eldan2016power,liang2016deep,poggio2016and,poole2016exponential,
	shaham2016provable,yarotsky2016error,daniely2017depth,safran2017depth}).
	These works mostly focus on depth separations: namely showing 
	that there are functions which can be expressed by a small network of a 
	given depth, but cannot be approximated by shallower networks, even if 
	their size is much larger. Perhaps the clearest manifestation of this is in 
	separating depth $2$ and depth $3$ networks: There are functions $f$ and 
	distributions $\mu$ on  
	$\reals^d$, which are
	\begin{itemize}
	\item Hard to approximate with a depth $2$ network: $\E_{\bx\sim\mu} 
	\left[\left(N_2(\bx)-f(\bx)\right)^2\right]\geq c$
	for some absolute $c>0$, using any depth $2$, width $\poly(d)$ network 
	$N_2(\bx):=\sum_{i=1}^{\poly(d)}u_i 
	\sigma(\bw_i^\top\bx+b_i)$ (for some parameters $\{v_i,\bw_i,b_i\}$ and 
	univariate activation function $\sigma$).
	\item Easy to approximate with a depth $3$ network: For any $\epsilon>0$, it 
	holds that $\E_{\bx\sim\mu} [(N_3(\bx)-f(\bx))^2]\leq 
	\epsilon$ (or sometimes even $\sup_{\bx}|N_3(\bx)-f(\bx)|\leq\epsilon$) for 
	some depth $3$, width $\poly(d,1/\epsilon)$ neural network 
	network $N_3(\bx):=\sum_{i=1}^{\poly(d,1/\epsilon)}u_i 
	\sigma\left(N_2^i(\bx)+b_i\right)$ (where each $N_2^i$ is a depth $2$, width $\poly(d,1/\epsilon)$ network, and $\sigma$ is a standard activation such as a ReLU). 
	\end{itemize}
	\citet{eldan2016power} (as well as a related construction in 
	\citet{safran2017depth}) prove such a lower bound unconditionally, 
	whereas \citet{daniely2017depth} show this with a simple proof, assuming 
	that the parameters of the network cannot be too large. Moreover, 
	these ``hard'' functions have a simple form: They are essentially radial 
	functions\footnote{\citet{eldan2016power} use a radial function. 
		\citet{daniely2017depth} use a function which is easily reduced to a radial 
		one -- see next paragraph.} of the form 
	$f(\bx)=g(\norm{\bx})$ for a univariate function 
	$g$. Such radial functions are of interest in learning theory, since there 
	are function classes that are essentially a mixture of radial functions 
	(e.g.\ Gaussian kernels), and they are essential primitives in expressing 
	functions which involve Euclidean distances. The intuition for the 
	above separations is that radial functions can be easily approximated with 
	depth $3$ networks, by first approximating the $\bx\mapsto 
	\norm{\bx}^2=\sum_i x_i^2$ function in the first layer, and then 
	approximating the univariate function $g(\sqrt{\cdot~})$ in the next layers. In contrast, 
	approximating high-dimensional radial functions with depth $2$ networks 
	appears to be difficult, since they are, in a sense, the furthest away from 
	functions which depend on only a single direction  
	(see \figref{fig:radialhigh}). 
	Overall, these results appear to provide a clear separation between the 
	required widths of depth $2$ and depth $3$ networks, in terms of the 
	dimension $d$.
	
	\begin{figure}
	\includegraphics[trim=2cm 0cm 2cm 0cm, clip=true, scale=0.85]{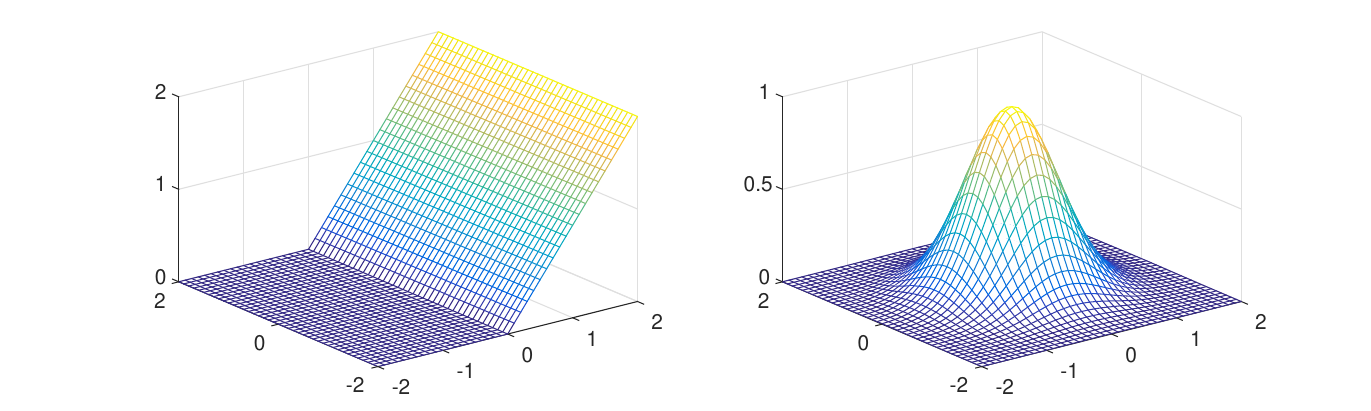}
	\caption{Left figure exemplifies the function $\bx\mapsto 
	\sigma(\inner{\bw,\bx})$ for $\sigma(z)=\max\{0,z\}$. Right figure 
	exemplifies a radial function. Intuitively, depth $2$ networks are linear 
	combinations of functions which vary only in one direction, whereas radial 
	functions can vary in all directions. Thus, in $d$ dimensions, it seems one 
	would need exponentially large (in $d$) depth $2$ networks to approximate 
	some radial functions well. To give an indirect analogy, it is known 
	that in order to uniformly approximate a $d$-dimensional unit sphere with a 
	polytope, one generally needs exponentially (in $d$) many facets 
	\citep{kochol2004note}.}\label{fig:radialhigh}
	\end{figure}
	
	However, a closer inspection of the constructions above	reveals that in 
	fact, this is not so clear. The reason is that the functions which are 
	shown to be provably hard for depth $2$ networks are rapidly oscillating, and 
	require a Lipschitz constant 
	(at least) polynomial in $d$ to even approximate: In 
	\citet{eldan2016power}, the function has the form 
	$\bx\mapsto \sum_{i=1}^{\Theta(d^2)}\epsilon_i \mathbbm{1}(\norm{\bx}\in 
	[a_i,b_i])$, over a distribution supported on $\reals^d$ (where 
	$\epsilon_i\in \{-1,+1\}$, $\mathbbm{1}$ is the indicator function, 
	and $[a_i,b_i]$ are disjoint 
	intervals in the range $\Theta(\sqrt{d})$). In \citet{daniely2017depth}, 
	the function used is easily reduced to $\sin(2\pi d^3\norm{\bx}^2)$ (see 
	proof of \thmref{thm:daniely_reduction}).
	Having such rapidly oscillating functions is not 
	always a natural regime, since we are often interested in functions whose 
	Lipschitz parameter is independent of the dimension. For 
	example, in learning theory, this is actually needed to obtain 
	dimension-free learnability results for convex functions \citep[Chapter 
	12]{shalev2014understanding}. Moreover, there is evidence that functions 
	which oscillate too rapidly can be computationally difficult for neural 
	networks to learn with standard gradient-based methods 
	(e.g., 
	\citep{song2017complexity,shalev2017failures,shamir2018distribution,abbe2018provable}),
	so Lipschitz functions are arguably more interesting from a learning 
	perspective. Overall, we are lead to the following natural 
	question: 
	\begin{quote}
	\emph{Can we show a depth $2$ vs.\ $3$ separation result in terms of the 
	dimension $d$, even for approximating $\Ocal(1)$-Lipschitz functions?}
	\end{quote}
	
	In other words, are there $\Ocal(1)$-Lipschitz functions which cannot be 
		approximated by depth $2$, width-$\poly(d)$	networks, but can be 
		approximated by depth $3$, width-$\poly(d)$ networks?
		
	To study this, we first notice that it is easy to reduce any hardness 
	result for approximating $L$-Lipschitz functions to accuracy $\epsilon$, to 
	hardness of approximating $1$-Lipschitz functions to accuracy 
	$\epsilon/L$, simply by scaling the functions by $1/L$. Moreover, we can 
	even reduce the hardness result to a $1$-Lipschitz function with accuracy 
	$\epsilon$, by dilating the measure we are using by a factor $L$ (see 
	Appendix \ref{app:reductions} for a formal statement). However, now the 
	lower 
	bounds require either the accuracy $\epsilon$ or the diameter 
	of the support of the distribution to scale polynomially with $d$. As a 
	result, when saying that the depth $2$ networks require width 
	super-polynomial in $d$, it is not clear whether the hardness really comes 
	from the dimension $d$, or perhaps from other parameters which 
	are being forced to scale with it, such as the 
	accuracy $\epsilon$. Thus, we rephrase our question as follows: 
	
	\begin{quote}
	\emph{Can we show 
	a depth $2$ vs.\ depth $3$ neural network separation result in terms of the 
	dimension $d$, for approximating $\Ocal(1)$-Lipschitz functions up to 
	constant accuracy $\epsilon$ on a 
	domain of bounded radius (all independent of $d$)?}
	\end{quote}
	
	The intuition described earlier (on the difficulty of approximating radial 
	functions in high dimensions) seems to suggest that the answer is 
	positive.
	
	\textbf{Our Results.} In this paper, we show that perhaps surprisingly, the 
	answer to the question above is actually negative (at least for radial 
	functions): For any constant $\epsilon$, it \emph{is} 
	possible to approximate radial functions using $\poly(d)$-width, depth $2$ 
	networks. More precisely, our upper bound on the required size is 
	$\exp\left(\Ocal\left(\epsilon^{-9}\log(d/\epsilon)\right)\right)$ (see 
		\thmref{thm:main}).
	We also complement this by showing that for constant dimension $d$, 
	approximation of any $\Ocal(1)$-Lipschitz radial function is possible with 
	$\poly(1/\epsilon)$-width, depth $2$ networks: Specifically, the bound is 
	$\exp\left(\Ocal\left(d\log(1/\epsilon)\right)\right)$ (see 
	\thmref{thm:poly_eps}). Both bounds are $L_{\infty}$-type 
	approximation results, with respect to the unit ball: Namely, given a 
	function $f$, we show how to find a neural network $n(\cdot)$ such that
	\[
	\sup_{\bx\in\unitball}|n(\bx)-f(\bx)|\leq\epsilon~,
	\]
	where $\unitball:=\{\bx\in\reals^d:\norm{\bx}\leq 1\}$. This is a stronger 
	approximation guarantee than $L_2$-type approximation guarantees (where we 
	bound
	$\E_{\bx\sim\mu}[(n(\bx)-f(\bx))^2]$ for some 
	distribution $\mu$ on $\unitball$), since a bound on the former implies a 
	bound on the latter. Furthermore, we show that any even radial monomial, namely a radial function of the form $ \bx\mapsto\norm{\bx}^{2k} $, for any fixed natural $ k $, can be approximated to accuracy $ \epsilon $ using a depth $2$ network of width polynomial in both $ d $ and $ 1/\epsilon $. Finally, we formally prove (using a reduction from 
	\citet{eldan2016power,daniely2017depth}, and using their 
	assumptions) that it is impossible to obtain a general polynomial dependence on 
	both $d$ and $1/\epsilon$ in our setting (see 	
	\thmref{thm:daniely_reduction} and \thmref{thm:unit_ball_reduction}). 
	Overall, these results show that to	approximate radial functions with 
	depth $2$ networks, their width \emph{can} be polynomial in either $d$ or $1/\epsilon$, but generally not in both. Putting this in the context of known depth separations, our results indicate that the difficulty in approximating the “hard” functions used in separating depth 2 from depth 3 stems from both the input dimension $d$ and the accuracy parameter $\epsilon$ simultaneously, and not from either one alone.
	
	It is interesting to note that such trade-offs between dimension and 
	accuracy also appear in very different areas of learning theory. For 
	example, consider the classic problem of agnostically learning halfspaces 
	up to excess error $\epsilon$ in $d$ dimensions (see for example 
	\citet{kalai2008agnostically}): It is folklore that for well-behaved input 
	distributions, one can learn a halfspace in runtime $\poly(1/\epsilon)$ for 
	constant $d$ (simply by creating an $\epsilon$-net 
	of all possible halfspaces, and picking the best one on a training data). 
	On the other hand, it is also known that one can learn in runtime 
	$\poly(d)$ for any constant $\epsilon$, at least for certain input 
	distributions \citep{kalai2008agnostically}. However, there is 
	evidence that being polynomial over both $d,1/\epsilon$ is not 
	possible in those settings \citep{klivans2014embedding}. 
	
	Finally, we emphasize that our results still do not 
	fully settle the question stated above, since there might be depth 
	separation results using functions which cannot be reduced to radial ones. 
	However, such results do not exist at the present time, and we believe that 
	our observations may also be relevant for more general families of 
	functions. In any case, we hope our paper would motivate and guide further 
	study of this question.

	\section{Main Results}
	
	In this section, we present our main results and the high-level proof 
	components. The remainder of the proofs are provided in 
	\secref{sec:proofs}. 
	
	\subsection{Approximation with Width $\poly(d)$ Networks}
	
	We first present our formal result, implying that radial functions can be 
	approximated with depth $2$, width $\poly(d)$ networks, to any constant 
	accuracy $\epsilon$. We prove this result for networks employing any 
	activation function $\sigma(\cdot)$ which satisfies the following mild 
	assumption (taken from \citet{eldan2016power}), which implies that the 
	activation can be used to approximate univariate functions well. This 
	assumption is satisfied for all standard activations, such as ReLU and 
	sigmoidal functions (see reference above for further discussion): 
	
	\begin{assumption}\label{asm:assumption1}
		Given the activation function $ \sigma $, there is a constant $ 
		c_{\sigma}\ge1 $ (depending only on $ \sigma $) such that
		the following holds: For any $ L $-Lipschitz function $ f : \reals\to\reals $ which is constant outside a bounded interval $ [-R,R] $, and for any $ \delta $, there exist scalars $ a,\set{\alpha_i,\beta_i,\gamma_i}_{i=1}^w $, where $ w\le c_{\sigma}\frac{RL}{\delta} $, such that the function
		\[
		h(x) = a +\sum_{i=1}^{w} \alpha_i\sigma(\beta_ix-\gamma_i)
		\]
		satisfies
		\[
		\sup_{x\in\reals} \abs{f(x) - h(x)} \le\delta.
		\]
	\end{assumption}
	Our main result for this subsection is the following:
	\begin{theorem}\label{thm:main}
		Suppose $ \sigma:\reals\to\reals $ satisfies Assumption 
		\ref{asm:assumption1}. Then for any $ \epsilon>0 $ and any $ 1 
		$-Lipschitz radial function $ f(\bx)=\varphi(\norm{\bx}) $, there 
		exists a depth $2$ neural network with $ \sigma $ activations and width $ w\le \exp\p{\Ocal\p{\epsilon^{-9}\log(d/\epsilon)}} $ satisfying
		\begin{equation*}
			\sup_{\bx\in 
			\unitball}\abs{\sum_{i=1}^{w}v_i\sigma\p{\bw_i^{\top}\bx+b_i}+b_0-f(\bx)}
			\le \epsilon,
		\end{equation*}
		where the big O notation hides a constant that depends solely on $ \sigma $.
	\end{theorem}
	
	We note that the $1/\epsilon^9$ exponent might be improvable to some 
	smaller polynomial in $1/\epsilon$ (see proof for details), but overall the 
	dependence on $1/\epsilon$ remains at least exponential. 
	
	The proof of this theorem requires several intermediate results about the 
	approximation capabilities of depth $2$ networks, some of which may be of 
	independent interest. The high-level strategy is the following: 
	\begin{itemize}
	\item First, we 
	consider depth $2$ networks $\sum_i v_i \sigma(\bw_i^\top\bx+b_i)$, where 
	$\sigma(z)=\exp(z)$ is the exponential function. Using properties of the 
	beta distribution, we show that if the 
	weights $\bw_i$ are drawn uniformly and independently from the unit sphere 
	(and $v_i,b_i$ are fixed appropriately), then the resulting network $N$ 
	satisfies $\E[N(\bx)]=F_d(\norm{\bx})$ for some complicated function $F_d$, 
	which 
	depends however only on the norm of $\bx$. Using concentration of measure, 
	we show that the above implies $N(\bx)\approx F_d(\norm{\bx})$ if the 
	width is sufficiently large (\thmref{thm:approx_conf_hypergeo_exp}). 
	\item Next, we use Assumption \ref{asm:assumption1} to show that we can 
	construct a bounded-width network $N(\cdot)$ with any $\sigma$-activation 
	(not just an 
	exponential one), such that $N(\bx)\approx F_d(\norm{\bx})$ 
	(\thmref{thm:approx_conf_hypergeo_sigma}). 
	\item Using a Taylor series argument, we show that a careful 
	linear combination of (not too many) scaled versions of $F_d(\norm{\bx})$ 
	allow us to 
	approximate any even monomial $\norm{\bx}^{2k}$ in the norm of $\bx$. Since 
	a linear combination of depth $2$ networks is still a depth $2$ network, 
	this implies that we can 
	approximate $\norm{\bx}^{2k}$ with some depth $2$ network, again with 
	bounded width (\thmref{thm:poly}). 
	\item Finally, we use a quantitative version of Weierstrass' approximation 
	theorem, to show that we can approximate any 
	Lipschitz radial function $\varphi(\norm{\bx})$ (where $\varphi$ is on $[0,1]$) 
	by a linear combination of even monomials (\lemref{lem:even_poly}). Again, 
	this implies that we can find a bounded-width depth $2$ network which approximates 
	this radial function well.
	\end{itemize}
	
	\subsection{Approximation with Width $\poly(1/\epsilon)$ Networks}
	
	Having considered depth $2$, width $\poly(d)$ networks (for constant 
	accuracy $\epsilon$), we now turn to consider the complementary setting, 
	where the dimension $d$ is fixed, and we show how Lipschitz radial 
	functions can be approximated by width $\poly(1/\epsilon)$ networks. This 
	setting is closer in spirit to universal approximation theorems for 
	depth $2$ networks (namely, on how such networks can approximate any 
	continuous function on a compact domain, if we allow exponential 
	dependencies on $d$). Unfortunately, most such theorems are not 
	quantitative in nature, and do not imply polynomial dependence on 
	$\epsilon$. A noteworthy exception is the line of work pioneered by Barron 
	(see \citet{barron1993universal}), which provide quantitative approximation 
	guarantees in terms of the width and moments of the Fourier transform of 
	the target function $f$. Our main technical contribution here is to show 
	how we can translate such moment-based bounds to a bound applicable to any 
	Lipschitz radial function. For concreteness, we will focus here on networks 
	employing the common ReLU activations (i.e.\ $\sigma(z)=\max\{0,z\}$), 
	although the technique is applicable more generally. We make use of the 
	following recent result from \citet[Theorem 2]{klusowski2018approximation}, 
	which provides an $L_{\infty}$ approximation guarantee for ReLU networks:
		\begin{theorem}[\cite{klusowski2018approximation}]\label{thm:barron}
			Let $ D=[-1,1]^d $. Suppose $ f $ admits a Fourier representation $ f(\bx)=\int_{\reals^d}\exp(i\inner{\bx,\omega})\Fcal(f)(\omega)d\omega $ and
			\[
				v_{f,2}=\int_{\reals^d} \norm{\omega}_1^2 \abs{\Fcal(f)(\omega)}d\omega < \infty.
			\]
			Then there exist depth $2$ ReLU networks $ f_n $, each of width 
			$ n+2 $ such that for all $ n $
			\begin{equation}\label{eq:barron}
				\sup_{\bx\in D}\abs{f(\bx)-f_n(\bx)} \le cv_{f,2}\sqrt{d+\log n}~n^{-1/2-1/d},
			\end{equation}
			for some universal constant $ c>0 $.
		\end{theorem}
		Note that in their original theorem statement, 
		\citet{klusowski2018approximation} define the ReLU networks as having 
		an additional linear $ \inner{a_0,\bx} $ term, which we for convenience 
		write as a sum of two ReLU neurons $ \inner{a_0,\bx} = 
		\relu{\inner{a_0,\bx}}+\relu{\inner{-a_0,\bx}} $ and thus omit it from 
		the theorem statement.
		
		We now turn to formally state the main result of this subsection.
		\begin{theorem}\label{thm:poly_eps}
			Suppose $ f(\bx)=\varphi(\norm{\bx}) $ is a $ 1 $-Lipschitz radial 
			function on $ \unitball $. Then there exists a depth $2$ ReLU neural 
			network $ N $, of width $ n=\exp\p{\Ocal\p{d\log\p{1/\epsilon}}} $ 
			such that
			\[
				\sup_{\bx\in\unitball} \abs{f(\bx)-N(\bx)}<\epsilon.
			\]
		\end{theorem}
		
		The proof (in \secref{sec:proofs}) utilizes \thmref{thm:barron}, with 
		the main challenge being that even for a $1$-Lipschitz radial $f$, the 
		coefficient $v_{f,2}$ might be unbounded. Instead, we consider a 
		smoothed approximation $g=f\star\gamma_{\epsilon^2/4d}$, where $ \star $ 
		is the convolution operation and $ \gamma_{\epsilon^2/4d} $ is the 
		Gaussian pdf with mean $ \bzero $ and covariance matrix $ 
		\frac{\epsilon^2}{4d}I $. Since $f$ is Lipschitz, this function is 
		$\Ocal(\epsilon)$-close to $f$ at any point $\bx$. Therefore, to 
		approximate $f$ well, it is sufficient to approximate $g$ well.  
		Moreover, since $g$ 
		represents a convolution with a smooth function, then it is smooth, and 
		therefore its Fourier transform has a rapidly decaying tail. This 
		implies that the coefficient $v_{g,2}$ is bounded (in a manner 
		exponential in $d$ but polynomial in $1/\epsilon$), and an application 
		of \thmref{thm:barron} implies the result.

	\subsection{Impossibility to Approximate with Width $\poly(d,1/\epsilon)$ 
	Networks}

		In this subsection, we complement our previous positive approximation 
		results with negative results. Specifically, we provide two lower 
		bounds, which imply that there are $1$-Lipschitz radial functions, 
		which cannot be approximated to accuracy $\epsilon$ on the unit ball 
		$\unitball$, using depth $2$, width $\poly(d,1/\epsilon)$ networks (see \figref{fig:inapprox}). In 
		a sense, this was already shown in 
		\citet{daniely2017depth,eldan2016power}, as discussed in the 
		introduction. However, a bit of work is needed to apply them to our 
		setting: For example, the result in \citet{eldan2016power} is for a 
		radial function, but not a Lipschitz one, and the result in 
		\citet{daniely2017depth} is not for a radial function. 
		
		Since our results are based on reductions from these papers, we need to 
		make similar assumptions. In particular, we need to require either 
		having an approximation on an unbounded domain, or that the 
		approximating network's parameters are at most exponential in 
		$d$. To the best of our knowledge, it remains a major open problem 
		to prove a depth separation result without either of these 
		two assumptions (namely, on a compact domain such as 
		$\unitball$, and without restrictions on the magnitude of the 
		parameters).

		\begin{theorem}\label{thm:daniely_reduction}
			The following holds for some positive universal
			constants $ c_1,c_2 $, and any depth $2$ network employing a ReLU 
			activation function. Consider the $ 1 $-Lipschitz function $ 
			f(\bx)=\frac{1}{2\pi d^3}\sin\p{2\pi d^3 \norm{\bx}_2^2} $ on $ 
			\unitball $. Suppose $ N $ is a depth $2$ network of width $ 
			w(d,1/\epsilon) $, with weights bounded by $ \frac{2^{d+1}}{2\pi 
			d^3} $, and satisfying $ \sup_{\bx\in 
			\unitball}\abs{N(\bx)-f(\bx)}\le\epsilon $ for any $ \epsilon>0 $ 
			and any $ d\ge2 $. Then for any $ d>c_1 $,
			\[
				w(d,101\exp(2)\pi^3d^3) \ge 2^{c_2d\log d}.
			\]
			In particular, depth $2$ networks of width $ \textnormal{poly}(d,1/\epsilon) $ cannot approximate $ f $ to accuracy $ \epsilon $.
			
		\end{theorem}
		
		We remark that the impossibility result provided in the theorem above is in terms of $L_{\infty}$-type approximation, namely 
		$\sup_{\bx}|n(\bx)-f(\bx)|$ rather than 
		$\E_{\bx}\left[n(\bx)-f(\bx)\right]^2$. This is for simplicity and to make the setting complementary to our positive results from earlier (however, extending it to $L_2$ approximation results is not too difficult). 
		\begin{theorem}\label{thm:unit_ball_reduction}
			The following holds for some positive universal
			constants $ c_1, c_2, c_3, c_4 $, and any network employing an activation function satisfying Assumptions \ref{asm:assumption1} and 2 in \citet{eldan2016power}.	Let $ f(\bx)=\max\set{0,-\norm{\bx}+1} $. For any $ d>c_1 $, there exists a continuous probability distribution on $ \reals^d $, such that for any $ \epsilon>0 $, and any depth $2$ neural network $ N $ satisfying $ \norm{N(\bx)-f(\bx)}_{L_2} \le \epsilon $ and having width $ w(d,1/\epsilon) $, it must hold that
			\[
				w(d,c_2d^6) \ge c_3\exp(c_4d).
			\]
			In particular, depth $2$ networks of width $ \textnormal{poly}(d,1/\epsilon) $ cannot approximate $ f $ to accuracy $ \epsilon $.
		\end{theorem}
		
		\begin{figure}\label{fig:inapprox}
			\includegraphics[width=0.5\columnwidth]{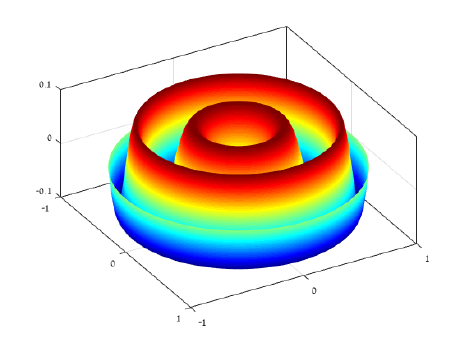}
			\includegraphics[width=0.5\columnwidth]{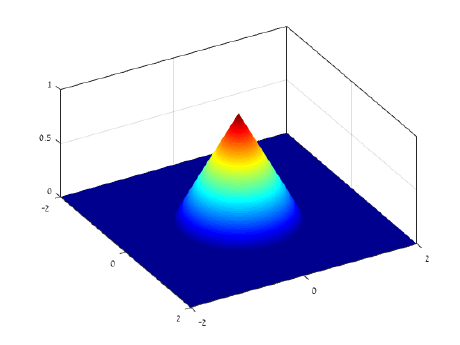}
			\caption{Left figure depicts a scaled version of the function from \thmref{thm:daniely_reduction}, which cannot be approximated on the unit ball $ \unitball $ by a depth $2$ network of width $ \poly(d,1/\epsilon) $ and weights bounded by $ \exp(\Ocal(d)) $. Right figure depicts the function from \thmref{thm:unit_ball_reduction}, which cannot be approximated by a depth $2$ network of width $ \poly(d,1/\epsilon) $ w.r.t.\ a measure supported on $ \reals^d $.}
		\end{figure}

	

	\section{Proofs}\label{sec:proofs}
	
	\subsection{Proof of \thmref{thm:main}}
	
	We begin by stating the following theorem, which establishes the capability 
	of exponential networks to approximate a particular radial 
	function, which we denote by $ F_d $. Our construction for approximating $ 
	F_d $ uses random weights, resulting in a random network which is 
	significantly easier to analyze when exponential activations are considered 
	(basically, since the exponent of a random variable $ X $ is its moment 
	generating function, which for many distributions is well-known and 
	studied).
	
	\begin{theorem}\label{thm:approx_conf_hypergeo_exp}
		For an integer $ k $, define $ k!!=\prod_{i=0}^{\left\lceil k/2 \right\rceil-1}(k-2i) $. For any $ 
		\epsilon>0 $ and natural $ d\ge2 $ there exists an exponential 
		depth $2$ neural network $ N(\bx)=\sum_{i=1}^{n}v_i 
		\exp\left(\bw_i^\top \bx\right)$ on $\reals^d$, of width $ 
		n=\ceil{\frac{36}{\epsilon^2}} $, hidden layer 
		weights $ \bw_i $ satisfying $ \bw_i\in\mathbb{S}^{d-1} $ (the unit 
		sphere), and $|v_i|\leq \frac{1}{n}$, such that
		\begin{equation*}
		\sup_{\bx\in \unitball}\abs{N(\bx)-F_d\p{\norm{\bx}}} \le 
		\epsilon,
		\end{equation*}
		where for all $ z\in[0,1] $,
		\begin{equation*}
		F_d(z) = \sum_{k=0}^{\infty}\frac{(d-2)!!}{(2k)!!(d+2k-2)!!}z^{2k}.
		\end{equation*}
	\end{theorem}
	The proof of \thmref{thm:approx_conf_hypergeo_exp} relies on the 
	observation that by drawing	$\bw_i$ uniformly from the unit 
	sphere, the neuron $\exp(\bw_i^\top\bx)$ has an expected value equal to 
	$F_d(\norm{\bx})$. Setting $v_i=\frac{1}{n}$ for all $i$, and sampling each 
	$\bw_i$ independently, we have from concentration of measure that the 
	resulting network $\frac{1}{n}\sum_{i=1}^{n}\exp(\bw_i^\top\bx)$ gradually 
	converges to this expected value, effectively 
	approximating $F_d$. Before we prove \thmref{thm:approx_conf_hypergeo_exp} 
	however, we would 
	need to evaluate the distribution of the dot product of such a random neuron with its input, as well as derive an equivalent representation of $ F_d $ which we will encounter when proving the theorem. To this end, we have the following two lemmas:
	\begin{lemma}\label{lem:dotprod_dist}
		Suppose $ \bx\in\reals^d $ such that $ \norm{\bx}=r $, and suppose $ 
		W\in\reals^d $ is distributed uniformly on the $ d $-dimensional unit 
		sphere. Then the random variable $ 
		X=\frac{1}{2r}W^{\top}\bx+\frac{1}{2} $ follows a $ 
		\textnormal{Beta}\p{\frac{d-1}{2},\frac{d-1}{2}} $ distribution.
	\end{lemma}
	
	\begin{proof}
		Since $ W $ is invariant to orthogonal transformations, we may assume 
		w.l.o.g.\ that $ \bx $ is of the form $ \bx=(r,0,\dots,0) $. That is, $ 
		W^{\top}\bx =W_1 r $, where $ W_1 $ is the first coordinate of $ W $. 
		Therefore to determine the distribution of $ W^{\top}\bx $, it suffices 
		to compute the probability of $ W_1 r $ falling in the interval $ 
		\pcc{-r,t} $ for $ t\in\pcc{-r,r} $, or equivalently, $ W_1 $ falling 
		in the interval $ \pcc{-1,\frac{t}{r}} $. Since $ W $ is distributed 
		uniformly on the unit sphere, this is proportional to the area of a 
		hyperspherical cap centered at $ \ba=(-1,0,\ldots,0) $, and defined by 
		$ \mathsf{S}^{d-1}(\ba,\theta)\coloneqq 
		\set{\bb\in\mathbb{S}^{d-1}:\arccos(\inner{\ba,\bb})\le\theta} $.
		This probability is given in terms of the regularized incomplete beta 
		function as
		\[
			I_{\sin^2\frac{\theta}{2}}\p{\frac{d-1}{2}, \frac{d-1}{2}}
		\]
		\citep[Lemma 2.3.15.]{leopardi2007distributing},
		where the spherical radius of the cap, $ \theta $, satisfies $ 
		\frac{t}{r}=\cos(\pi-\theta)=-\cos(\theta) $. Elementary trigonometry 
		reveals that under this condition, it must hold that $ 
		\sin^2\p{\frac{\theta}{2}}=\frac{t}{2r}+\frac{1}{2} $, namely we have
		\[
			\pr\pcc{W^{\top}\bx\in\pcc{-r,t}} = 
			I_{\frac{t}{2r}+\frac{1}{2}}\p{\frac{d-1}{2}, \frac{d-1}{2}},
		\]
		implying that
		\begin{align*}
			\pr\pcc{X\in\pcc{0,\frac{t}{2r}+\frac{1}{2}}} &= 
			\pr\pcc{\frac{1}{2r}W^{\top}\bx+\frac{1}{2}\in\pcc{0,\frac{t}{2r}+\frac{1}{2}}}
			 \\
			&= \pr\pcc{W^{\top}\bx\in\pcc{-r,t}} \\
			&= I_{\frac{t}{2r}+\frac{1}{2}}\p{\frac{d-1}{2}, \frac{d-1}{2}},
		\end{align*}
		i.e.,\ by the change of variables $ x=\frac{t}{2r}+\frac{1}{2} $ we have
		\[
			\pr\pcc{X\in\pcc{0,x}} = I_{x}\p{\frac{d-1}{2}, \frac{d-1}{2}}.
		\]
		It follows immediately that $ X $ is $ 
		\textnormal{Beta}\p{\frac{d-1}{2},\frac{d-1}{2}} $ distributed, 
		concluding the proof of the lemma.
	\end{proof}
	
	\begin{lemma}\label{lem:closed_form_fd}
		We have
		\begin{equation*}
			F_d(z)\coloneqq \sum_{n=0}^{\infty}\frac{(d-2)!!}{(2n)!!(d+2n-2)!!}z^{2n} = \exp\p{-z} 
			\p{\sum_{k=0}^{\infty}\p{\prod_{j=0}^{k-1}\frac{(d-1)/2+j}{d-1+j}}\frac{2^k}{k!}z^k}.
		\end{equation*}
	\end{lemma}
	
	\begin{proof}
		Letting $ (x)_k=\prod_{j=0}^{k-1}(x+j) $, we compute
		\begin{align*}
			& \exp\p{-z} 
			\p{\sum_{k=0}^{\infty}\p{\prod_{j=0}^{k-1}\frac{(d-1)/2+j}{d-1+j}}\frac{2^k}{k!}z^k}
			 \\
			=&\p{\sum_{k=0}^{\infty}\frac{(-1)^k}{k!}z^k}\p{\sum_{k=0}^{\infty}\frac{2^k\p{\frac{d-1}{2}}_k}{k!\p{d-1}_k}z^k}
			 \\
			=& 
			\sum_{n=0}^{\infty}\p{\sum_{k=0}^{n}\frac{(-1)^{n-k}2^k\p{\frac{d-1}{2}}_k}{(n-k)!k!(d-1)_k}}z^k.
		\end{align*}
		Since $ (-1)^{n-k}=(-1)^k $ for even $ n $ and $ (-1)^{n-k}=-(-1)^k $ 
		for odd $ n $, it suffices to show that for any natural $ n $ and any 
		integer $ d\ge2 $,
		\begin{equation*}
			a_n(d)\coloneqq\sum_{k=0}^{n}\frac{(-1)^{k}2^k\p{\frac{d-1}{2}}_k}{(n-k)!k!(d-1)_k}
			 = \begin{cases}
			0 & n\text{ odd}\\
			\frac{(d-2)!!}{n!!(d+n-2)!!} & n\text{ 
			even}
			\end{cases}.
		\end{equation*}
		We rewrite
		\begin{equation}\label{eq:sum_to_hgeo}
			a_n(d) = 
			\frac{1}{n!}\sum_{k=0}^{n}(-1)^{k}\binom{n}{k}\frac{\p{\frac{d-1}{2}}_k}{(d-1)_k}2^k
			 = \frac{1}{n!}\pFq{2}{1}\p{-n,\frac{d-1}{2};d-1;2},
		\end{equation}
		where
		\[
			\pFq{2}{1}(a,b;c;z)=\sum_{n=0}^{\infty}\frac{(a)_n(b)_n}{(c)_n}\frac{z^n}{n!}
		\]
		is the Gauss hypergeometric function. Using Euler's integral formula 
		for the Gauss hypergeometric function \citep[p. 65, Theorem 
		2.2.1]{andrews1999special} yields
		\begin{equation}\label{eq:euler_int}
			\pFq{2}{1}\p{-n,\frac{d-1}{2};d-1;2}=\frac{\Gamma(d-1)}{\Gamma\p{\frac{d-1}{2}}^2}\int_{0}^{1}t^{\frac{d-3}{2}}(1-t)^{\frac{d-3}{2}}(1-2t)^ndt.
		\end{equation}
		Simplifying the integral in \eqref{eq:euler_int}, we substitute $ 
		t=0.5-x $, $ dt=-dx $ to get
		\begin{align}
			\int_{0}^{1}t^{\frac{d-3}{2}}(1-t)^{\frac{d-3}{2}}(1-2t)^ndt &= 
			\int_{-0.5}^{0.5}(0.5-x)^{\frac{d-3}{2}}(0.5+x)^{\frac{d-3}{2}}(2x)^ndx
			 \nonumber\\
			&= 2^n\int_{-0.5}^{0.5}(0.25-x^2)^{\frac{d-3}{2}}x^ndx. 
			\label{eq:substitution}
		\end{align}
		Clearly, the integrand in \eqref{eq:substitution} is an odd function 
		when $ n $ is odd, therefore $ a_n(d)=0 $ for any odd $ n $. For even $ 
		n $, integration by parts of $ (0.25-x^2)^{\frac{d-3}{2}}x $ and $ 
		x^{n-1} $ reveals that
		\begin{equation}\label{eq:recursive_rule}
			\int_{-0.5}^{0.5}(0.25-x^2)^{\frac{d-3}{2}}x^ndx = 
			\frac{n-1}{d-1}\int_{-0.5}^{0.5}(0.25-x^2)^{\frac{d-1}{2}}x^{n-2}dx.
		\end{equation}
		Recursively applying the relation in \eqref{eq:recursive_rule} yields
		\begin{equation}\label{eq:recursive_result}
			\int_{-0.5}^{0.5}(0.25-x^2)^{\frac{d-3}{2}}x^ndx = 
			\frac{n-1}{d-1}\cdot\frac{n-3}{d+1}\cdot\ldots\cdot\frac{1}{d+n-3}\int_{-0.5}^{0.5}(0.25-x^2)^{\frac{d-1}{2}}dx.
		\end{equation}
		Substituting $ x=0.5-t $ back in the integral in 
		\eqref{eq:recursive_result} gives
		\begin{align}
			\int_{-0.5}^{0.5}(0.25-x^2)^{\frac{d-1}{2}}dx &= 
			\int_{0}^{1}t^{\frac{d+n-3}{2}}(1-t)^{\frac{d+n-3}{2}}dt 
			\nonumber\\ &= B\p{\frac{d+n-1}{2},\frac{d+n-1}{2}} \nonumber\\ &= 
			\frac{\Gamma\p{\frac{d+n-1}{2}}^2}{\Gamma(d+n-1)} \label{eq:beta},
		\end{align}
		where $ B(x,y) $ denotes the Beta function. Finally, substituting our 
		calculations from Equations 
		(\ref{eq:beta},\ref{eq:recursive_result},\ref{eq:substitution},\ref{eq:euler_int}) in \eqref{eq:sum_to_hgeo}, and using the identities $ \Gamma(z+1)=z\Gamma(z) $ which holds for any real $ z\ge0 $, and $ \Gamma(z+1)=z!=z!!(z-1)!! $ which holds for any integer $ z $, we have
		\begin{align*}
			a_n(d) &= \frac{1}{n!}\cdot\frac{\Gamma(d-1)}{\Gamma\p{\frac{d-1}{2}}^2}\cdot2^n\cdot\frac{n-1}{d-1}\cdot\frac{n-3}{d+1}\cdot\ldots\cdot\frac{1}{d+n-3}\cdot\frac{\Gamma\p{\frac{d+n-1}{2}}^2}{\Gamma(d+n-1)} \\
			&= \frac{(n-1)!!}{n!!(n-1)!!}\cdot\frac{\Gamma(d-1)}{\Gamma\p{\frac{d-1}{2}}^2}\cdot2^n\cdot\frac{1}{d-1}\cdot\frac{1}{d+1}\cdot\ldots\cdot\frac{1}{d+n-3}\cdot\frac{\Gamma\p{\frac{d+n-1}{2}}^2}{\Gamma(d+n-1)} \\
			&= \frac{1}{n!!}\cdot\frac{\Gamma(d-1)}{\Gamma\p{\frac{d-1}{2}}^2}\cdot2^n\cdot\frac{\p{\frac{d-1}{2}}^2}{d-1}\cdot\frac{\p{\frac{d+1}{2}}^2}{d+1}\cdot\ldots\cdot\frac{\p{\frac{d+n-3}{2}}^2}{d+n-3}\cdot\frac{\Gamma\p{\frac{d-1}{2}}^2}{\Gamma(d+n-1)} \\
			&= \frac{1}{n!!}\cdot\frac{\Gamma(d-1)}{\Gamma(d+n-1)}\cdot (d-1)(d+1)\ldots(d+n-3) \\
			&= \frac{1}{n!!}\cdot\frac{(d-2)!!(d-3)!!}{(d+n-2)!!(d+n-3)!!}\cdot \frac{(d+n-3)!!}{(d-3)!!} \\
			&= \frac{(d-2)!!}{n!!(d+n-2)!!}.
		\end{align*}
	\end{proof}
	We are now ready to prove \thmref{thm:approx_conf_hypergeo_exp}.
	
	\begin{proof}[Proof of \thmref{thm:approx_conf_hypergeo_exp}]
		Consider a depth $2$ network of width $ n $, where $ n $ is to be 
		determined later, with exponential activations, $ 0 $ bias terms in the 
		hidden layer, equal weights of $ 1/n $ in the output neuron, and where 
		the weights of each hidden neuron are sampled i.i.d.\ uniformly at 
		random from the unit hypersphere $ \mathbb{S}^{d-1} $.
		Fix $ r $ such that $ \norm{\bx}=r $, then we have from 
		\lemref{lem:dotprod_dist} that the network computes the random function
		\begin{equation}\label{eq:random_net_function}
			N\p{\bx}=\sum_{i=1}^{n}\frac{1}{n}\exp\p{W_i^{\top}\bx} = 
			\frac{1}{n}\sum_{i=1}^{n}\exp\p{2rX_i-r}
		\end{equation}
		where $ X_i\sim \textnormal{Beta}(\frac{d-1}{2},\frac{d-1}{2}) $ are 
		i.i.d.
		Taking expectation in \eqref{eq:random_net_function} yields.
		\begin{align*}
			\E_{X_1,\dots,X_n}\pcc{N\p{\bx}} &= 
			\frac{1}{n}\sum_{i=1}^{n}\E_{X_1,\dots,X_n}\pcc{\exp\p{2rX_i-r}} \\ 
			&= 
			\frac{\exp\p{-r}}{n}\sum_{i=1}^{n}\E_{X_1,\dots,X_n}\pcc{\exp\p{2rX_i}} 
			\\
			&= \frac{\exp\p{-r}}{n}\sum_{i=1}^{n}\E_{X_1}\pcc{\exp\p{2rX_1}}.
		\end{align*}
		Letting $ t=2r $ gives
		\begin{equation}\label{eq:mgf}
			\E\pcc{N\p{\bx}} = \exp\p{-t/2}\E_{X_1}\pcc{\exp\p{tX_1}}.
		\end{equation}
		Conveniently, the expectation in the right hand side of \eqref{eq:mgf} 
		is exactly the moment generating function of a $ 
		\textnormal{Beta}(\frac{d-1}{2},\frac{d-1}{2}) $ random variable, given 
		by
		\[
		\E\pcc{\exp\p{tX}} = 
		\sum_{k=0}^{\infty}\p{\prod_{j=0}^{k-1}\frac{(d-1)/2+j}{d-1+j}}\frac{t^k}{k!},
		\]
		\citep{gupta2004handbook}. By virtue of \lemref{lem:closed_form_fd}, \eqref{eq:mgf} therefore reduces to
		\begin{equation}\label{eq:expected_neuron}
			\E\pcc{N\p{\bx}} = \exp\p{-\norm{\bx}} 
			\sum_{k=0}^{\infty}\p{\prod_{j=0}^{k-1}\frac{(d-1)/2+j}{d-1+j}}\frac{2^k}{k!}\norm{\bx}^k = F_d(\norm{\bx}).
		\end{equation}
		To convert the above expectation equality to a uniform convergence bound we shall use a Rademacher complexity argument. We have that the approximation error is
		\begin{equation}\label{eq:uniformconv}
		\sup_{\bx\in \unitball}|N(\bx)-F_d(\norm{\bx})|
		~=~ \sup_{\bx\in \unitball}\left|\frac{1}{n}\sum_{i=1}^{n}\exp(W_i^\top 
		\bx)-\E[\exp(W^\top\bx)]\right|.
		\end{equation}
		This is equivalent to bounding the uniform convergence of the function 
		class $\Fcal:=\{W\mapsto \exp(W^\top\bx):\bx\in \unitball\}$, whose 
		values are 
		bounded in $[\exp(-1),\exp(1)]$. By standard Rademacher complexity 
		arguments, it is well-known that this is upper bounded by 
		$\Ocal(\sqrt{\log(1/\delta)/n})$ with probability at least $1-\delta$. 
		Specifically, letting $\phi(z):=\frac{\exp(z)-1}{\exp(1)-1}$, we can 
		rewrite \eqref{eq:uniformconv} as
		\[(\exp(1)-1)\cdot\sup_{\bx\in 
		\unitball}\left|\frac{1}{n}\sum_{i=1}^{n}\phi(W_i^\top 
				\bx)-\E[\phi(W^\top\bx)]\right|~.
		\]
		Defining the function class $\Fcal':=\{W\mapsto 
		W^\top\bx:\bx\in\unitball)\}$, we can upper bound the above 
		(with probability at least $1-\delta$ over the sampling of 
		$W_1,\ldots,W_n$) by
		\[
		\left(\exp(1)-1)\right)\cdot\left(2R_n(\phi\circ 
		\Fcal'(W_1,\ldots,W_n))+\sqrt{\frac{2\log(2/\delta)}{n}}\right)~,
		\]
		where 
		$R_n(\phi\circ \Fcal'(W_1,\ldots,W_n)):=\E\sup_{f\in\Fcal'}\
		\left|\frac{1}{n}\sum_{i=1}^{n}\sigma_i \phi(f(W_i))\right|$ is the 
		(empirical) Rademacher complexity of $\Fcal$, and the expectation is 
		over $\sigma_1,\ldots,\sigma_n$ which are sampled independently and 
		uniformly from $\{-1,+1\}$ (see \citet[Theorem 
		3.2]{boucheron2005theory}). Since $W^\top\bx$ takes values in 
		$[-1,+1]$, $\phi$ is $1$-Lipschitz in that domain and $\phi(0)=0$, we 
		can upper bound the above by
		\[
		\left(\exp(1)-1)\right)\cdot\left(2R_n(
		\Fcal'(W_1,\ldots,W_n))+\sqrt{\frac{2\log(2/\delta)}{n}}\right)~,
		\]		
		(see \citet[Theorem 3.3]{boucheron2005theory}). Finally, since $F'$ 
		consists of $1$-Lipschitz linear functions over the unit ball, we have 
		that $R_n(F'(W_1,\ldots,W_n))\leq \sqrt{1/n}$
		(see \citet[Corollary 4.3]{boucheron2005theory}). Overall, we get that 
		\eqref{eq:uniformconv} is at most 
		$(\exp(1)-1)\left(\frac{2}{\sqrt{n}}+\sqrt{\frac{2\log(2/\delta)}{n}}
		\right)$. Picking $\delta=3/4$, this can be upper bounded by 
		$6/\sqrt{n}$. In particular, this means that there \emph{exist} some realizations of
		$W_1,\ldots,W_n$ such that \eqref{eq:uniformconv} is at most 
		$6/\sqrt{n}$. In other words, for any $\epsilon>0$, if we set $n\geq 
		\frac{36}{\epsilon^2}$, we have a depth $2$ Linear network of width $n$ 
		which approximates $F_d(\norm{\bx})$ up to error $\epsilon$.
		
	\end{proof}

	Albeit useful for establishing \thmref{thm:approx_conf_hypergeo_exp}, 
	exponential activations are uncommon in practice. To translate 
	\thmref{thm:approx_conf_hypergeo_exp} to work with more commonly used 
	activations, we utilize the universality of activations satisfying 
	Assumption \ref{asm:assumption1} to approximate an exponential function on 
	a bounded domain to arbitrary accuracy, resulting in a network 
	approximating $ F_d $ for a wide family of activation functions. More 
	formally, we have the following theorem:
	
	\begin{theorem}\label{thm:approx_conf_hypergeo_sigma}
		Suppose $ \sigma:\reals\to\reals $ is an activation satisfying 
		Assumption \ref{asm:assumption1}. Then for any $ \epsilon>0 $ and 
		natural $ d\ge2 $ there exists a depth $2$ neural network $ 
		N:\reals^d\to\reals $ with $ \sigma $ activations of width at most $ 
		c_{\sigma}\epsilon^{-3} $, satisfying $ \sup_{\bx\in \unitball}\abs{N(\bx)-F_d\p{\norm{\bx}}} \le \epsilon $, where $ c_{\sigma}>0 $ depends solely on $ \sigma $.
	\end{theorem}
	
	\begin{proof}[Proof of \thmref{thm:approx_conf_hypergeo_sigma}]
		First, invoke \thmref{thm:approx_conf_hypergeo_exp} to obtain a width $ 
		n=\ceil{144\epsilon^{-2}} $ exponential network satisfying 
		\begin{equation}\label{eq:good_hypergeo_approx}
			\sup_{\bx\in \unitball}\abs{N(\bx)-F_d(\norm{\bx})}\le 
			\frac{\epsilon}{2}.
		\end{equation} 
		Next, using Assumption 
		\ref{asm:assumption1}, we obtain a depth $2$ $ \sigma $ network 
		approximating the exponential $ z\mapsto\exp(z) $ on the unit interval 
		$ [0,1] $, having width at most $ c_{\sigma}e/\epsilon $. Denote this 
		network as $ N_{\exp} $, we construct a $ \sigma $ network 
		approximating $ F_d $ as follows: For each hidden weight $ \bw_i $ of $ 
		N $, we take a copy of $ N_{\exp} $ and feed it with $ 
		\inner{\bw_i,\bx} $ to obtain $ N_{\exp}(\inner{\bw_i,\bx}) $. Note 
		that $ N_{\exp}(\inner{\bw_i,\bx}) $ is a depth $2$ $ \sigma $ network 
		since the linear transformation $ \inner{\bw_i,\bx} $ can be simulated 
		by modifying the hidden layer of $ N_{\exp} $ to compute it exactly. 
		Defining the network $ \tilde{N}(\bx) = \frac{1}{n}\sum_{i=1}^{n} 
		N_{\exp}(\inner{\bw_i,\bx}) $, which is also a depth $2$ network of width 
		$ c_{\sigma}\epsilon^{-3} $ as a weighted combination of networks 
		(absorbing 
		any absolute constants into $ c_{\sigma} $). We now 
		compute using \eqref{eq:good_hypergeo_approx} for any $ \bx\in 
		\unitball $
		\begin{align*}
			\abs{\tilde{N}(\bx) - F_d(\norm{\bx})} &=
			\abs{\tilde{N}(\bx) - N(\bx) + N(\bx) - F_d(\norm{\bx})} \\
			&\le \abs{\tilde{N}(\bx) - N(\bx)} +\abs{N(\bx) - F_d(\norm{\bx})} 
			\\
			&\le \abs{\frac{1}{n}\sum_{i=1}^{n} N_{\exp}(\inner{\bw_i,\bx}) - 
			\frac{1}{n}\sum_{i=1}^{n} \exp(\inner{\bw_i,\bx})} + 
			\frac{\epsilon}{2} \\
			&\le \frac{1}{n}\abs{\sum_{i=1}^{n} \p{N_{\exp}(\inner{\bw_i,\bx}) - 
			\exp(\inner{\bw_i,\bx})}} + \frac{\epsilon}{2} \\
			&\le \frac{1}{n}\sum_{i=1}^{n}\abs{N_{\exp}(\inner{\bw_i,\bx}) - 
			\exp(\inner{\bw_i,\bx})} + \frac{\epsilon}{2} \\
			&\le \frac{1}{n}\sum_{i=1}^{n}\frac{\epsilon}{2} + 
			\frac{\epsilon}{2} = \epsilon.
		\end{align*}
		Where we note that the boundedness of the weights of the hidden layer 
		of $ N $ and the Cauchy-Schwarz inequality guarantee that we remain in 
		the relevant approximation domain of $ \exp(\cdot) $, as $ 
		\inner{\bw_i,\bx} \le \norm{\bw_i}\cdot\norm{\bx}\le1 $.
	\end{proof}
	
	\thmref{thm:approx_conf_hypergeo_sigma} allows us to approximate the family 
	of functions $ F_d(\norm{\cdot}) $ efficiently using depth $2$ networks with 
	a variety of activations. The following theorem utilizes 
	\thmref{thm:approx_conf_hypergeo_sigma} to approximate even radial 
	monomials. i.e.,\ radial functions of the form $ \norm{\bx}^{2k} $ for some 
	natural $ k $.
	
	\begin{theorem}\label{thm:poly}
		Suppose $ \sigma:\reals\to\reals $ satisfies Assumption 
		\ref{asm:assumption1}. Then for any $ \epsilon>0 $ and any natural $ 
		k\ge1 $, there exists a depth $2$ neural 
		network with $ \sigma $ activations of width $ 
		n=\exp\p{\Ocal\p{k^2\log\p{d/\epsilon}}} $ satisfying
		\begin{equation*}
			\sup_{\bx\in 
			\unitball}\abs{\sum_{i=1}^{n}v_i\sigma\p{\bw_i^{\top}\bx+b_i}+b_0-\norm{\bx}^{2k}}
			 \le \epsilon,
		\end{equation*}
		where the big O notation hides a constant that depends solely on $ \sigma $.
	\end{theorem}
	Interestingly, apart from its role in 
	proving \thmref{thm:main}, 
	\thmref{thm:poly} also shows the existence of a family of functions that 
	are approximable to accuracy $ \epsilon $ using width polynomial in 
	\emph{both} $ d $ and $ 1/\epsilon $; for any \emph{fixed} $ k $, the 
	radial polynomial $ \norm{\bx}^{2k} $ can be approximated by a width $ 
	\exp\p{\Ocal\p{k^2\log\p{d/\epsilon}}}=\textnormal{poly}(d,1/\epsilon) $ 
	network. Before delving into the proof of \thmref{thm:poly}, however, we 
	will first need the following lemma, which will utilize the power-series 
	representation of $ F_d $ to approximate polynomials of even degree. Note 
	that at this point, the question of approximation is now reduced to a one 
	dimensional problem, since approximating a radial $ \varphi(\norm{\bx}) $ 
	using linear combinations of $ F_d(\norm{\bx}) $ is equivalent to 
	approximating $ \varphi(z) $ using linear combinations of $ F_d(z) $.
	
	\begin{lemma}\label{lem:taylor_approx}
		Suppose $ f(z)=\sum_{k=0}^{\infty} \alpha_{2k}z^{2k} $ converges 
		uniformly for all $ z\in[0,1] $, where $ \alpha_{2k}\neq0 $ is non 
		increasing. Then for sufficiently small $ \epsilon>0 $ and any $ n>0 $, 
		there exist $ b_0,\dots,b_n,c_1,\ldots,c_n\in\reals$, and a universal constant $ c>0 $ 
		such that 
		\begin{equation}\label{eq:poly_approx}
			\sup_{z\in\pcc{0,1}}\abs{b_0+\sum_{k=1}^{n}b_kf(c_kz) - z^{2n}} \le 
			\epsilon,
		\end{equation} 
		where $ \abs{c_k}\le1 $ and $ \abs{b_k} \le 
		\alpha_{2n}^{-1}\p{\frac{2}{\epsilon}}^{cn^2} $ for all $ 
		k\in\set{0,\dots,n} $.
	\end{lemma}
	The proof of \lemref{lem:taylor_approx} relies on the observation that taking an appropriately chosen linear combination of the form $ f(\eta z),f(\eta^{2} 
	z),\dots,f(\eta^{n} z) $ for some $ \eta>0 $ and presenting it as a power-series, results in all the coefficients of $ z^{2k} $ for $ k<n $ being exactly zero, the coefficient of $ z^{2n} $ being $ 1 $, and the remaining coefficients all decaying rapidly to $ 0 $ as $ \eta\to0 $.
	\begin{proof}
		Let $ p(z)=\sum_{k=0}^{n}p_{2k}z^{2k} $ be some even polynomial, and 
		consider the set of functions $$ f(\eta z),f(\eta^{2} 
		z),\dots,f(\eta^{n} z). $$ These have the following expansions:
		\begin{align*}
			f(\eta z) &=\sum_{k=0}^{\infty} \alpha_{2k}\eta^{2k}z^{k} \\
			f(\eta^{2} z) &=\sum_{k=0}^{\infty} \alpha_{2k}\eta^{4k}z^{k} \\
			&~~\vdots\\
			f(\eta^{n} z) &=\sum_{k=0}^{\infty} \alpha_{2k}\eta^{2kn}z^{k}.
		\end{align*}
		Equating the coefficients $ t_{2i} $, $ i=1,2,\dots $ in $ 
		\sum_{i=1}^{\infty}t_{2i}z^{2i} $, the expansion of $ 
		\sum_{k=1}^{n}b_kf(\eta^{k}z) $, to the coefficients of $ p(z) $, we 
		obtain the matrix equality
		\begin{equation}\label{eq:matrix_equation}
			A\cdot V(\eta)\cdot\bb = \bp,
		\end{equation}
		where $ A=\text{diag}(\balpha) $ is a diagonal matrix with the 
		coefficients $ \balpha=\alpha_2,\dots,\alpha_{2n} $ on its main 
		diagonal, $ V(\eta) $ is the Vandermonde matrix given by
		\[
			V(\eta)=\p{\begin{matrix}
				\eta & \eta^{2} & \dots & \eta^{n} \\
				\eta^{2} & \eta^{4} & \dots & \eta^{2n} \\
				\vdots & \vdots & \ddots & \vdots \\
				\eta^{n} & \eta^{2n} & \dots & \eta^{n^2}
			\end{matrix}},
		\]
		$ \bb=\p{b_1,\dots,b_n} $ and $ \bp = \p{p_2,\dots,p_{2n}}$. Since $ 
		\alpha_{2k}\neq 0 $, and since $ V(\eta) $ is invertible for small 
		enough $ \eta $, \eqref{eq:matrix_equation} can be rearranged to
		\begin{equation*}
			\bb = V(\eta)^{-1}A^{-1}\bp.
		\end{equation*}
		Letting $ b_0 =\alpha_0\sum_{k=1}^{n}b_k $, we have that the 
		coefficients $ t_{2i} $ up to degree $ 2n $ agree with $ p(z) $, thus 
		to establish \eqref{eq:poly_approx}, it remains to bound the tail of 
		the expansion for degrees $ >2n $. To this end, we will first bound 
		each $ b_k $ for $ k=1,\dots,n $.
		We have from H\"older's inequality for all $ b_k $
		\begin{equation}\label{eq:holder}
			\abs{b_k} = 
			\abs{\inner{V(\eta)_k^{-1},(\alpha_2p_2,\dots,\alpha_{2n}p_{2n})}} 
			\le 
			\norm{V(\eta)_k^{-1}}_1\cdot\norm{(\alpha_2^{-1}p_2,\dots,\alpha_{2n}^{-1}p_{2n})}_{\infty},
		\end{equation}
		where $ V(\eta)_k^{-1} $ is the $ k $-th row of $ V(\eta)^{-1} $, given by
		\[
			V(\eta)_{ij}^{-1} = \begin{cases}
				(-1)^{j-1}\frac{\sum\limits_{1\le m_1<\ldots<m_{n-j}\le n \atop 
				m_1,\dots,m_{n-j}\neq i}\eta^{m_1}\cdot\ldots\cdot 
				\eta^{m_{n-j}}}{\eta^{i}\prod\limits_{1\le m\le n\atop m\neq 
				i}(\eta^{m}-\eta^{i})} & 1\le j< n \\
				\frac{1}{\eta^{i}\prod\limits_{1\le m\le n\atop m\neq 
				i}(\eta^{m}-\eta^{i})} & j=n
			\end{cases}
		\]
		\citep{macon1958inverses}. Bounding $ V(\eta)_{ij}^{-1} $, we begin with the denominator to obtain
		for $ \eta^{-1}\ge n $ that if $ i<m $ then
		\[
			\abs{\eta^{m}-\eta^{i}} = \eta^{i}\abs{1-\eta^{m-i}} \ge 
			\p{1-\frac{1}{n}}\eta^{i}.
		\]
		Otherwise, if $ i>m $ then
		\[
			\abs{\eta^{m}-\eta^{i}} = \eta^{m}\abs{1-\eta^{i-m}} \ge 
			\p{1-\frac{1}{n}}\eta^{m}.
		\]
		Therefore
		\begin{align*}
			\eta^{i}\prod\limits_{1\le m\le n\atop m\neq i}(\eta^{m}-\eta^{i}) &\ge 
			\eta^{i}\p{1-\frac{1}{n}}^{n-i}\eta^{i(n-i)} \prod\limits_{1\le m<i} 
			\p{1-\frac{1}{n}}\eta^{m} \\
			&= \p{1-\frac{1}{n}}^{n-1}\eta^{i}\eta^{i(n-i)}\eta^{i(i-1)/2} \\
			&\ge \exp\p{-1}\eta^{-i^2/2+in+i/2}.
		\end{align*}
		Thus
		\begin{align*}
			\frac{1}{\eta^{i}\prod\limits_{1\le m\le n\atop m\neq 
			i}(\eta^{m}-\eta^{i})} \le e\eta^{i^2/2-in-i/2}.
		\end{align*}
		For the numerator we have
		\[
			\sum\limits_{1\le m_1<\ldots<m_{n-j}\le n \atop 
			m_1,\dots,m_{n-j}\neq i}\eta^{m_1}\cdot\ldots\cdot \eta^{m_{n-j}} \le 
			\prod_{k=1}^{n-j}\eta^{k}\le \eta^{(n-j)},
		\]
		which also holds for $ j=n $. Hence we have
		\[
			\abs{V(\eta)_{ij}^{-1}}\le e\eta^{i^2/2-in-i/2+n-j},
		\]
		implying the $ 1 $-norm of the $ k $-th row is upper bounded by
		\[
			\norm{V(\eta)_{k}^{-1}}_1\le 2e\eta^{k^2/2-kn-k/2}.
		\]
		Combining the above with \eqref{eq:holder}, we obtain the upper bound 
		for some $ c_2>0 $
		\begin{align}\label{eq:expansion_term_bound}
			\abs{b_k} &\le 
			2e\max_{j\in[n]}\abs{\alpha_{2j}^{-1}p_{2j}}\eta^{k^2/2-kn-k/2} 
			\nonumber\\
			&\le 2e\alpha_{2n}^{-1}\eta^{k^2/2-kn-k/2} \\ &\le 
			\alpha_{2n}^{-1}\p{\frac{2}{\epsilon}}^{c_2n^2} \nonumber.
		\end{align}
		In general, the coefficient of the term $ z^{2i} $ for $ i>n $ in the 
		expansion of $ b_0+\sum_{k=1}^{n}b_kf(\eta^kz) $ is given by
		\begin{equation*}
			t_{2i}=\alpha_{2i}\sum_{k=1}^{n}b_k\eta^{2ki} .
		\end{equation*}
		Taking the absolute value and combining with 
		\eqref{eq:expansion_term_bound}, we get
		\begin{align*}
			\abs{t_{2i}}&\le \abs{\alpha_{2i}\sum_{k=1}^{n}b_k\eta^{2ki}} \\
			&\le 
			2e\alpha_{2i}\alpha_{2n}^{-1}\sum_{k=1}^{n}\eta^{k^2/2-kn-k/2+2ki} \\
			&\le 2e\sum_{k=1}^{n}\eta^{k^2/2-k/2}(\eta^{n-2i})^{-k} \\
			&\le 2e\sum_{k=1}^{n}(\eta^{n-2i})^{-k} \\
			&\le 4e\eta^{(2i-n)}.
		\end{align*}
		Finally, letting $ \eta=\min\set{0.5,1/n,\epsilon/8e} $ (note that this 
		also entails $ c_k=\eta^{k}\le1 $), we can bound the tail as follows
		\begin{align*}
			\abs{\sum_{i=n+1}^{\infty}t_{2i}x^{2i}} &\le 
			\sum_{i=n+1}^{\infty}\abs{t_{2i}}
			\le 4e\sum_{i=n+1}^{\infty}\eta^{(2i-n)} \\
			&= 4e\sum_{i=1}^{\infty}\eta^{2i} \le 8e\eta
			\le \epsilon.
		\end{align*}
	\end{proof}
	With the help of \lemref{lem:taylor_approx}, we now turn to prove 
	\thmref{thm:poly}.
	\begin{proof}[Proof of \thmref{thm:poly}]
		First, note that the family of functions $ F_d(z) $ satisfy the 
		assumptions in \lemref{lem:taylor_approx} for any $ d\ge2 $, as readily 
		seen by their definition. Now, letting
		\[
			\alpha_{2k}(d)=\frac{(d-2)!!}{(2k)!!(d+2k-2)!!},
		\]
		we obtain from \lemref{lem:taylor_approx} that
		\begin{equation}
			\sup_{z\in\pcc{0,1}}\abs{b_0+\sum_{k=1}^{n}b_kF_d\p{c_kz} - z^{2n}} 
			\le \frac{\epsilon}{2}, \label{eq:taylor_lemma_approx}
		\end{equation}
		for coefficients $ \abs{c_k}\le1 $ and $ b_k $ satisfying
		\begin{align}\label{eq:poly_coeff}
			\abs{b_k}&\le 
			\alpha_{2n}^{-1}(d)\p{\frac{2}{\epsilon}}^{\Ocal(n^2)} \nonumber\\
			&= 
			\frac{(2n)!!(d+2n-2)!!}{(d-2)!!}\p{\frac{2}{\epsilon}}^{\Ocal(n^2)}\nonumber\\
			 &= 
			2^nn!d(d+2)\ldots(d+2n-2)\p{\frac{2}{\epsilon}}^{\Ocal(n^2)}\nonumber\\
			 &\le
			d(d+2)\ldots(d+2n-2)\p{\frac{2}{\epsilon}}^{\Ocal(n^2)}
		\end{align}
		To bound $ d(d+2)\ldots(d+2n-2) $, observe that $ (d+2n-2)\le 
		d^{\log_2n+1} $ for any $ n\ge1 $ and $ d\ge2 $, thus
		\begin{equation*}
			d(d+2)\ldots(d+2n-2) \le d^{n\log_2n+n}
		\end{equation*}
		Plugging the above in \eqref{eq:poly_coeff} yields
		\[
			\abs{b_k}\le d^{n\log_2n+n}\p{\frac{2}{\epsilon}}^{\Ocal(n^2)}.
		\]
		It now remains to approximate the function $ F(\bx) = b_0 + 
		\sum_{k=1}^{n}b_kF_d\p{c_k\norm{\bx}} $ to accuracy $ \epsilon/2 $ 
		(note that the $ b_0 $ is trivial, as it can be easy to simulate with a constant neuron). To this end, invoke 
		\thmref{thm:approx_conf_hypergeo_sigma} with a desired accuracy of $ 
		\frac{\epsilon}{2n\abs{b_n}} $, to obtain a network $ N $ approximating 
		$ F_d(\norm{\bx}) $. We stress that such approximation of $ 
		F_d(\norm{\bx}) $ is obtained for any $ \bx\in\unitball $, and since we 
		have $ \abs{c_k}\le1 $ we are guaranteed to remain in the relevant 
		domain. Taking $ n $ such copies of $ N $, we obtain a width $ 
		8c_{\sigma}n^3\abs{b_n}^3\epsilon^{-3} = 
		\exp\p{\Ocal\p{n^2\log\p{d/\epsilon}}} $ network
		\[
			N^{\prime}(\bx) = b_0 + \sum_{k=1}^{n} b_kN(c_k\bx)
		\] 
		approximating $ F(\bx) $, since
		\begin{align}
			\sup_{\bx\in\unitball}\abs{F(\bx)-N^{\prime}(\bx)}
			&=
			\sup_{\bx\in\unitball}\abs{b_0 + 
			\sum_{k=1}^{n}b_kF_d(c_k\norm{\bx}) - b_0 - 
			\sum_{k=1}^{n}b_kN(c_k\norm{\bx})} \nonumber\\ &\le 
			\sup_{\bx\in\unitball}\sum_{k=1}^{n}\abs{b_k}\abs{F_d(c_k\norm{\bx})
			 - N(c_k\norm{\bx})} \nonumber\\
			&\le \sum_{k=1}^{n}\abs{b_n} \frac{\epsilon}{2n\abs{b_n}} = 
			\frac{\epsilon}{2}. \label{eq:Fd_approx}
		\end{align}
		Combining Equations (\ref{eq:taylor_lemma_approx}) and 
		(\ref{eq:Fd_approx}), we conclude that
		\[
			\sup_{\bx\in\unitball}\abs{N^{\prime}(\bx)-\norm{\bx}^{2n}} \le 
			\sup_{\bx\in\unitball}\abs{N^{\prime}(\bx) - F(\bx)} + \abs{F(\bx) 
			- \norm{\bx}^{2n}} \le \frac{\epsilon}{2} + \frac{\epsilon}{2} = 
			\epsilon.
		\]
	\end{proof}
	
	Before we can prove \thmref{thm:main}, it only remains that we first prove 
	the following lemma, establishing quantitative bounds on the ability of 
	even polynomials having degree $ n $ to approximate arbitrary $ 1 
	$-Lipschitz functions in $ [0,1] $, while having bounded coefficients. More 
	formally, we have the following lemma:
	
	\begin{lemma}\label{lem:even_poly}
		Let $ f:[0,1]\to\reals $ be a $ 1 $-Lipschitz function. Then for any $ 
		\epsilon>0 $, there exists an even polynomial $ p $ of degree $ 
		n=2\ceil{4\epsilon^{-3}} $ such that
		\[
			\sup_{z\in[0,1]}\abs{p(z)-f(z)} \le \epsilon,
		\]
		and where the coefficients of $ p $, denoted $ p_2,p_4,\ldots,p_{n} $ are upper 
		bounded by $ 2^n $.
	\end{lemma}
	
	We remark that the $-3$ exponent in the result can possibly be improved 
	somewhat, but this will not change the exponential dependence on 
	$1/\epsilon$ in our main theorem. 
	
	The following proof follows along a similar line as the proof provided by 
	S. Bernstein for Weierstrass' approximation theorem (see \citet[Thm. 
	2.7]{koralov2007theory} for the proof), albeit we also bound the magnitude 
	of the coefficients of the approximating polynomial.
	\begin{proof}
		Let $ f:[0,1]\to\reals $ be $ 1 $-Lipschitz. First, by approximating $ 
		f(z)-f(0.5) $ instead, we may assume w.l.o.g.\ that $ f(0.5)=0 $ (adding 
		the zero degree polynomial $ f(0.5) $ to our approximation once 
		obtained). Extend $ f $ to an even function on $ [-1,1] $ given by
		\[
			f(z)=\begin{cases}
				f(z) & z\ge0 \\
				f(-z) & z<0
			\end{cases}.
		\]
		Letting $ g(z)=f(2z-1) $, we linearly shift $ f $ to the unit interval 
		where $ g(z) $ is $ 2 $-Lipschitz. 
		Define the $ n+1 $ Bernstein basis polynomials of degree $ n $ as 
		\[
			b_{\nu,n}(z)=\binom{n}{\nu}z^{\nu}(1-z)^{n-\nu}.
		\]
		It is a well known fact that these polynomials form a partition of 
		unity for any $ n $:
		\begin{equation}\label{eq:partition_of_unity}
			\sum_{\nu=0}^{n}b_{\nu,n}(z)=1.
		\end{equation}
		Define the $ n $-th Bernstein polynomial approximation of $ g $ as
		\[
			P_n(g)(z) = \sum_{\nu=0}^{n} g\p{\frac{\nu}{n}}b_{\nu,n}(z).
		\]
		We compute using \eqref{eq:partition_of_unity} 
		\begin{align}
			\abs{P_n(g)(z)-g(z)}  &= \abs{\sum_{\nu=0}^{n} 
			\binom{n}{\nu}z^{\nu}(1-z)^{n-\nu} \p{g\p{\frac{\nu}{n}} - g(z)}} 
			\nonumber\\
			&\le \sum_{\nu:\abs{\frac{\nu}{n}-z}<\frac{\epsilon}{4}} 
			\binom{n}{\nu}z^{\nu}(1-z)^{n-\nu} \abs{g\p{\frac{\nu}{n}} - g(z)} 
			\label{eq:below_eps}\\
			&~~~~~~~~+\sum_{\nu:\abs{\frac{\nu}{n}-z}\ge\frac{\epsilon}{4}} 
			\binom{n}{\nu}z^{\nu}(1-z)^{n-\nu} \abs{g\p{\frac{\nu}{n}} - g(z)}. 
			\label{eq:above_eps}
		\end{align}
		Since $ g $ is $ 2 $-Lipschitz, we have that $ \abs{\frac{\nu}{n}-z} < 
		\frac{\epsilon}{4} $ implies $ \abs{g(\frac{\nu}{n})-g(z)} < 
		\frac{\epsilon}{2} $, thus (\ref{eq:below_eps}) is upper bounded by
		\[
			\frac{\epsilon}{2}\sum_{\nu:\abs{\frac{\nu}{n}-z}<\frac{\epsilon}{4}}
			 \binom{n}{\nu}z^{\nu}(1-z)^{n-\nu} \le \frac{\epsilon}{2}.
		\]
		Recalling that $ g(0.25)=g(0.75)=0 $, we have from Lipschitzness that $ 
		\sup_{z\in[0,1]} \abs{g(z)}\le0.5 $. Therefore (\ref{eq:above_eps}) is 
		upper bounded by
		\begin{equation}\label{eq:binomial_dist}
			\sum_{\nu:\abs{\frac{\nu}{n}-z}\ge\frac{\epsilon}{4}} 
			\binom{n}{\nu}z^{\nu}(1-z)^{n-\nu}.
		\end{equation}
		Observing \eqref{eq:binomial_dist} is exactly $ 
		\pr\pcc{\abs{\frac{X_n}{n}-z} \ge \frac{\epsilon}{4}} $, where $ 
		X_n\sim B(n,z) $ is binomially distributed. Using Chebyshev's 
		inequality we obtain
		\[
			\pr\pcc{\abs{\frac{X_n}{n}-z} \ge \frac{\epsilon}{4}} \le 
			\frac{16nz(1-z)}{n^2\epsilon^2} \le \frac{4}{n\epsilon^2}.
		\]
		Letting $ n = 2\ceil{4\epsilon^{-3}} $ entails 
		(\ref{eq:above_eps}) is upper bounded by $ \frac{\epsilon}{2} $, 
		yielding
		\[
			 \sup_{z\in[0,1]} \abs{\sum_{\nu=0}^{n} 
			 f\p{\frac{2\nu}{n}-1}b_{\nu,n}(z) - f(2z-1)} = \sup_{z\in[0,1]} 
			 \abs{\sum_{\nu=0}^{n} g\p{\frac{\nu}{n}}b_{\nu,n}(z) - g(z)} \le 
			 \frac{\epsilon}{2} + \frac{\epsilon}{2} = \epsilon,
		\]
		or equivalently by changing $ z=\frac{1+t}{2} $,
		\begin{equation}\label{eq:bernstein_approx}
			\sup_{t\in[-1,1]} \abs{\sum_{\nu=0}^{n} 
			f\p{\frac{2\nu}{n}-1}b_{\nu,n}\p{\frac{1+t}{2}} - f(t)} \le 
			\epsilon.
		\end{equation}
		Denote $ p(t) = \sum_{\nu=0}^{n} 
		f\p{\frac{2\nu}{n}-1}b_{\nu,n}\p{\frac{1+t}{2}} $. We shall now bound 
		the coefficients of the approximating polynomial $ p(t) $. We have
		\begin{align*}
			p(t) &= 
			\sum_{\nu=0}^{n}f\p{\frac{2\nu}{n}-1}\binom{n}{\nu}\p{\frac{t+1}{2}}^{\nu}\p{1-\frac{t+1}{2}}^{n-\nu}
			 \\
			&= 
			\sum_{\nu=0}^{n}f\p{\frac{2\nu}{n}-1}\binom{n}{\nu}\p{\frac{1+t}{2}}^{\nu}\p{\frac{1-t}{2}}^{n-\nu}
			 \\
			&= 
			\frac{1}{2^n}\sum_{\nu=0}^{n}f\p{\frac{2\nu}{n}-1}\binom{n}{\nu}\p{1+t}^{\nu}\p{1-t}^{n-\nu}.
		\end{align*}
		To upper bound the coefficients, observe that taking the absolute value 
		of $ f\p{\frac{2\nu}{n}-1} $ and substituting $ 1-t $ with $ 1+t $ will 
		result in a polynomial with only positive coefficients, upper bounding 
		the ones of $ p(t) $. Therefore
		\begin{align*}
			\frac{1}{2^n}\sum_{\nu=0}^{n}\abs{f\p{\frac{2\nu}{n}-1}}\binom{n}{\nu}\p{1+t}^{\nu}\p{1+t}^{n-\nu}
			 &= 
			\frac{1}{2^n}\sum_{\nu=0}^{n}\abs{f\p{\frac{2\nu}{n}-1}}\binom{n}{\nu}\p{1+t}^{n}
			 \\
			&\le \frac{\p{1+t}^{n}}{2^{n+1}}\sum_{\nu=0}^{n}\binom{n}{\nu} \\
			&= \frac{1}{2}(1+t)^n
		\end{align*}
		Clearly, the coefficients of $ \frac{1}{2}(1+t)^n $ are upper bounded 
		by $ 2^n $. Finally, consider the even polynomial
		\[
			\tilde{p}(t) = \frac{1}{2} \p{p(t) + p(-t)}.
		\]
		Its even coefficients are equal to those of $ p $ and are thus bounded 
		by $ 2^n $. Moreover, we have
		\begin{equation}\label{eq:minus_berstein}
			\sup_{t\in[-1,1]} \abs{p(-t) - f(-t)} \le \epsilon.
		\end{equation}
		By virtue of $ f $ being even we have $ f(t)=\frac{1}{2}\p{f(t)+f(-t)} 
		$, and by Equations (\ref{eq:bernstein_approx}) and 
		(\ref{eq:minus_berstein}) we get for any $ t\in[-1,1] $
		\begin{align*}
			\abs{\tilde{p}_n(t)-f(t)} &= \abs{\frac{1}{2}\p{p(t)+p(-t)} - 
			\frac{1}{2}\p{f(t)+f(-t)}} \\
			&\le \frac{1}{2}\abs{p(t) - f(t)} + \frac{1}{2}\abs{p(-t) - f(-t)} 
			\le \epsilon,
		\end{align*}
		concluding the proof of the lemma.
	\end{proof}
	
	We are finally ready to prove \thmref{thm:main}.
	
	\begin{proof}[Proof of \thmref{thm:main}]
		From \lemref{lem:even_poly}, we have an even polynomial $ 
		p(z)=\sum_{k=0}^{n/2}p_{2k}z^{2k} $ of degree $ n = 2\ceil{32\epsilon^{-3}} $, such that
		\[
			\sup_{z\in[0,1]}\abs{p(z) - \varphi(z)} \le \frac{\epsilon}{2},
		\]
		thus also
		\begin{equation}\label{eq:main_poly}
			\sup_{\bx\in\unitball}\abs{p(\norm{\bx}) - f(\bx)} = 
			\sup_{\bx\in\unitball} \abs{p(\norm{\bx}) - \varphi(\norm{\bx})} 
			\le \frac{\epsilon}{2}.
		\end{equation}
		Invoke \thmref{thm:poly} $ \frac{n}{2} $ times to approximate each of $ 
		\norm{\bx}^2,\norm{\bx}^4,\ldots,\norm{\bx}^{n} $ to accuracy $ 
		\frac{\epsilon}{n2^n} $, using $ \frac{n}{2} $ depth $2$ networks $ N_k 
		$, $ k=1,2,\ldots, n/2 $, with $ \sigma $ activations of width $ w\le 
		c_{\sigma}\p{\frac{2nd2^n}{\epsilon}}^{\Ocal(n^2)} = 
		c_{\sigma}\p{\frac{2d}{\epsilon}}^{\Ocal(n^3)} $. Thus obtaining for 
		any $ k\in\set{1,2,\ldots,n/2} $
		\begin{equation}\label{eq:monomial_approx}
			\sup_{\bx\in \unitball}\abs{N_k(\bx) - \norm{\bx}^{2k}} \le 
			\frac{\epsilon}{n2^n}.
		\end{equation}
		Consider the depth $2$ $ \sigma $ network $ N $ concatenating the 
		networks $ p_{2k}\cdot N_k $, having output bias of $ p_0 $ and having 
		width 
		\[
			w^{\prime}\le 
			c_{\sigma}\frac{n}{2}\p{\frac{2d}{\epsilon}}^{\Ocal(n^3)} = \exp\p{\Ocal\p{\epsilon^{-9}\log(d/\epsilon)}}.
		\]
		We compute for any $ \bx\in\unitball $
		\begin{align*}
			\abs{N(\bx) - f(\bx)} &\le \abs{N(\bx) - p(\norm{\bx})} + 
			\abs{p(\norm{\bx}) - f(\bx)} \\
			&= \abs{p_0 + \sum_{k=1}^{n/2} p_{2k}N_k(\bx) - \sum_{k=0}^{n/2} 
			p_{2k}\norm{\bx}^{2k}} + \abs{p(\norm{\bx}) - f(\bx)} \\
			&\le \sum_{k=1}^{n/2} \abs{p_{2k}}\abs{N_k(\bx) - \norm{\bx}^{2k}} 
			+ \abs{p(\norm{\bx}) - f(\bx)}.
		\end{align*}
		From Equations (\ref{eq:main_poly}) and (\ref{eq:monomial_approx}), the 
		above is upper bounded by
		\[
			\sum_{k=1}^{n/2} 2^n\frac{\epsilon}{n2^n} + \frac{\epsilon}{2} = 
			\frac{\epsilon}{2} + \frac{\epsilon}{2} = \epsilon.
		\]
		The proof of \thmref{thm:main} is complete.
	\end{proof}

		\subsection{Proof of \thmref{thm:poly_eps}}
			Let $ f(\bx)=\varphi(\norm{\bx}) $ be $ 1 $-Lipschitz on $ 
			\unitball $. By setting the bias term of the output neuron of the 
			approximating depth $2$ network to $ b_0=f(\bzero) $, we may assume 
			w.l.o.g.\ that $ f(\bzero)=0 $ to begin with. Moreover, since we do 
			not care about the approximation attained on $ 
			\reals^d\setminus\unitball $, we may set $ f(\bx)=0 $ for any $ 
			\bx\in\reals^d\setminus\unitball $.
			
			Now, instead of uniformly approximating $ f $ directly, we can 
			approximate a smoothed $ \epsilon/2 $-approximation of it attained 
			by $ g = f \star \gamma_{\epsilon^2/4d}$, where $ \star $ is the 
			convolution operation and $ \gamma_{\epsilon^2/4d} $ is the Gaussian 
			density function with mean $ \bzero $ and covariance matrix $ 
			\frac{\epsilon^2}{4d}I $. Equivalently, we can define $g$ as
			\[
			g(\bx) = \E_{\bz}f(\bx+\bz)
			\]
			where $\bz$ is distributed according to $\gamma_{\epsilon^2/4d}$. We 
			note that this is a uniform $\epsilon/2$ approximation of $f$, 
			since $$|g(\bx)-f(\bx)|\leq \E_{\bz}|f(\bx+\bz)-f(\bx)|\leq 
			\E_{\bz}\norm{\bz}~\leq 
			\sqrt{\E_{\bz}\norm{\bz}^2}~=~\sqrt{\frac{\epsilon^2}{4}}=\frac{\epsilon}{2}.$$
			Since smooth functions have well-behaved Fourier transforms, this 
			will make the use of \thmref{thm:barron} much more convenient. We 
			thus have that attaining a uniform $ \epsilon/2 $-approximation of 
			$ g $ on $ \unitball $ will suffice to finish the proof.
			
			We begin by upper bounding $ \norm{f}_{L_1} $. Since $ f(\bzero)=0 
			$ and $ f $ is $ 1 $-Lipschitz, we have that $ \norm{f}_{L_1}\le 
			\int_{\unitball} d\bx = \mu_d(\unitball) $, where $ \mu_d $ denotes the $ d $-dimensional Lebesgue measure. Consequentially, since an $ L_1 $ upper bound implies a similar upper bound on the $ 
			L_{\infty} $ norm of the Fourier transform, we have that $ 
			\hat{f}(\omega)\le \mu_d(\unitball) $ for any $ \omega\in\reals^d 
			$. Since the Fourier transform of a Gaussian pdf is another 
			Gaussian with inverse variance, we have from the 
			convolution-multiplication theorem that $ \hat{g}(\omega) = 
			\hat{f}(\omega)\exp\p{-\epsilon^2\norm{\omega}^2/8d} $, for all $ 
			\omega\in\reals^d $. We thus compute
			\begin{align}
				v_{g,2} &= \int_{\reals^d} \norm{\omega}_1^2 
				\abs{\Fcal(g)(\omega)}d\omega \nonumber\\
				&\le 
				\int_{\reals^d}d^{0.5}\norm{\omega}_2^2\abs{\hat{f}(\omega)}\abs{\exp\p{-\epsilon^2\norm{\omega}_2^2/8d}}d\omega
				 \nonumber\\
				&\le d^{0.5}\mu_d(\unitball) 
				\int_{r=0}^{\infty}\int_{\mathbb{S}^{d-1}_r} \norm{\omega}_2^2 
				\exp\p{-\epsilon^2\norm{\omega}_2^2/8d}d\omega dr \nonumber\\
				&= d^{0.5}\mu_d(\unitball)\int_{r=0}^{\infty}\int_{\mathbb{S}^{d-1}_r} 
				r^2 \exp\p{-\epsilon^2 r^2/8d}d\omega dr \nonumber\\
				&= 
				d^{0.5}\mu_d(\unitball)\mu_{d-1}(\mathbb{S}_1^{d-1})\int_{r=0}^{\infty}
				 r^{d+1} \exp\p{-\epsilon^2 r^2/8d} dr \nonumber\\
				&= 
				\frac{1}{2}d^{0.5}\mu_d(\unitball)\mu_{d-1}(\mathbb{S}_1^{d-1})\int_{-\infty}^{\infty}
				 \abs{r}^{d+1} \exp\p{-\epsilon^2 r^2/8d}dr \nonumber\\
				&= 
				\frac{1}{2}d^{0.5}\mu_d(\unitball)\mu_{d-1}(\mathbb{S}_1^{d-1})\frac{1}{\sqrt{\pi}}2^{(d+1)/2}
				 \Gamma\p{\frac{d}{2}+1}\p{\frac{2\sqrt{d}}{\epsilon}}^{d+1} 
				\label{eq:normal_abs_moment}\\
				&\le \frac{\pi^{d-0.5}}{\Gamma(d/2)}d^{0.5}2^{1.5(d+1)}\p{\frac{\sqrt{d}}{\epsilon}}^{d+1} 
				\label{eq:ball_vol_and_area}\\
				&\le \pi^{d-0.5}d^{0.5}\p{\frac{2e}{d}}^{d/2-1}2^{1.5(d+1)}\p{\frac{\sqrt{d}}{\epsilon}}^{d+1} 
				\label{eq:gamma_ineq}\\
				&= 2^{2d+0.5}\exp\p{0.5d-1}\pi^{d-0.5}d^2\epsilon^{-d-1}
				\nonumber\\
				&=
				\exp\p{\Ocal\p{d\log\p{1/\epsilon}}}
				\label{eq:fourier_bound}
			\end{align}
			where \eqref{eq:normal_abs_moment} is due to the absolute moments 
			of a normal variable $ X $ with mean $ 0 $ and standard deviation $ 
			\sigma $ satisfying $ 
			\E\pcc{\abs{X}^{d+1}}=\sigma^{d+1}\frac{2^{(d+1)/2}\Gamma\p{d/2+1}}{\sqrt{\pi}} $ (see \citet[Eq. (18)]{winkelbauer2012moments}), \eqref{eq:ball_vol_and_area} is due to $ \mu_d(\unitball)=\frac{\pi^{d/2}}{\Gamma(d/2+1)} $ and $ \mu_{d-1}(\mathbb{S}_1^{d-1}) = \frac{2\pi^{d/2}}{\Gamma(d/2)} $, and \eqref{eq:gamma_ineq} is due to the inequality $ \Gamma(z)\ge \p{\frac{z}{e}}^{z-1} $.
			
			We now split our analysis into two cases, depending on the value of 
			$ c $, the constant guaranteed from \thmref{thm:barron}. In both 
			cases we will need the following:
			\begin{claim}\label{clm:sup_ineq}
				We have
				\[
					\sup_{x,d\in[3,\infty) \times [2,\infty)} 
					\sqrt{\frac{d+\log dx^3}{dx}}<1.
				\]
			\end{claim}
			The claim is a straightforward result derived by computing the 
			partial derivatives of the left hand side and showing it is 
			monotonically decreasing for any $ x $ and $ d $ in its domain, and 
			therefore its proof is omitted.
			
			Begin with assuming $ c\le1 $. Then substituting $ 
			n=8dv_{g,2}^3/\epsilon^3 $ in \eqref{eq:barron}, we get
			\begin{align*}
				\sup_{\bx\in D} \abs{g(\bx)-g_n(\bx)} &\le 
				cv_{g,2}\frac{\sqrt{d+\log n}}{n^{1/2+1/d}} \\
				&\le v_{g,2}\sqrt{\frac{d+\log n}{n}} \\
				&= 
				\frac{\epsilon}{2}\sqrt{\frac{d+\log(8dv_{g,2}^3/\epsilon^3)}{2dv_{g,2}/\epsilon}}
				 \\
				&< \frac{\epsilon}{2},
			\end{align*}
			where the last inequality is due to Claim \ref{clm:sup_ineq} and 
			the assumption that $ v_{g,2}/\epsilon \ge 1.5 $, which will always 
			hold for small enough $ \epsilon>0 $ since $ v_{g,2} $ is always 
			finite and positive for a non-constant $ 
			f=g\star\gamma_{\epsilon^2/4d} $.
			
			For the second case, assume $ c>1 $. Then choosing $ 
			n=8dc^3v_{g,2}^3/\epsilon^3 $ we similarly have
			\begin{align*}
				\sup_{\bx\in[-1,1]^d} \abs{g(\bx)-g_n(\bx)}
				&\le cv_{g,2}\sqrt{\frac{d+\log n}{n}} \\
				&= 
				\frac{\epsilon}{2}\sqrt{\frac{d+\log(8dc^3v_{g,2}^3/\epsilon^3)}{2dcv_{g,2}/\epsilon}}
				 \\
				&< \frac{\epsilon}{2},
			\end{align*}
			where likewise, the last inequality uses Claim \ref{clm:sup_ineq} 
			and the assumptions that $ v_{g,2}/\epsilon \ge 1.5 $ and $ c>1 $.
			
			We conclude using \eqref{eq:fourier_bound} that $ g $ can be $ \epsilon/2 $-approximated using a 
			depth $2$ ReLU network of width $$ 2 + 
			dv_{g,2}^3\max\{1,c^3\}/\epsilon^3 
			=\exp\p{\Ocal\p{d\log\p{1/\epsilon}}}, $$ completing the proof of 
			\thmref{thm:poly_eps}.

		\subsection{Proof of \thmref{thm:daniely_reduction}}
			Our proof essentially reduces the assumptions in the theorem statement to those of \citet[Example 2]{daniely2017depth}, who showed that any depth $2$ ReLU network which approximates the non-radial function $ 
			\sin\p{\pi d^3\inner{\bx_1,\bx_2}} $ to an expected accuracy of at most $ \frac{1}{50\exp(2)\pi^2} $ with respect to the uniform 
			distribution on $ \mathbb{S}^{d-1}\times\mathbb{S}^{d-1} $, while having weights bounded by $ 2^d $, necessarily has width at least $ 2^{\Omega(d\log d)} $.
			
			Suppose that $ f $ is approximable to accuracy $ \epsilon $ using a 
			depth $2$ network $ N $ of width $ w(d,1/\epsilon) $, having weights bounded by $ \frac{2^{d+1}}{2\pi d^3} $. i.e.\ suppose that
			\[
				\sup_{\bx\in \unitball}\abs{N(\bx) - f(\bx)} \le \epsilon.
			\]
			Then in particular, we can choose $ \epsilon = 
			\frac{1}{101\exp(2)\pi^3d^3} $ to have a width $ w(d, 101\exp(2)\pi^3d^3) $ 
			network satisfying
			\begin{equation}\label{eq:100pi}
				\sup_{\bx\in \unitball}\abs{N(\bx) - f(\bx)} < 
				\frac{1}{100\exp(2)\pi^3d^3}.
			\end{equation}
			Now, let $ \tilde{f}(\bx)=2\pi d^3f(\bx) $, and let $ 
			\tilde{N}(\bx)=2\pi d^3N(\bx) $, which is also a depth $2$ neural 
			network with weights bounded by $ 2^{d+1} $, since the scaling 
			factor of $ 2\pi d^3 $ can be simulated by multiplying the weights 
			of the output neuron of $ N $ by $ 2\pi d^3 $. We have using 
			\eqref{eq:100pi} that 
			\[
				\sup_{\bx\in \unitball}\abs{\tilde{N}(\bx) - \tilde{f}(\bx)} = 
				\sup_{\bx\in \unitball}\abs{2\pi d^3N(\bx) - 2\pi d^3f(\bx)} = 
				2\pi d^3\cdot\sup_{\bx\in \unitball}\abs{N(\bx) - f(\bx)} < 
				\frac{1}{50\exp(2)\pi^2},
			\]
			By taking a network $ N^{\prime} $ which is identical to $ 
			\tilde{N} $ except for having its first layer weights (excluding 
			bias terms) halved, i.e.\ bounded by $ 2^d $, we have
			\[
				\sup_{\bx\in 2\unitball}\abs{N^{\prime}(\bx) - 
				\tilde{f}\p{\frac{1}{2}\bx}} < \frac{1}{50\exp(2)\pi^2}.
			\]
			which implies that
			\[
				\sup_{\bx_1,\bx_2\in \unitball}\abs{N^{\prime}(\bx_1+\bx_2) - 
				\tilde{f}\p{\frac{1}{2}(\bx_1+\bx_2)}} < \frac{1}{50\exp(2)\pi^2},
			\]
			as well as
			\begin{equation}\label{eq:supsum}
				\sup_{\bx_1,\bx_2\in 
				\mathbb{S}^{d-1}}\abs{N^{\prime}(\bx_1+\bx_2) - 
				\tilde{f}\p{\frac{1}{2}(\bx_1+\bx_2)}} < \frac{1}{50\exp(2)\pi^2},
			\end{equation}
			since further restricting the domain cannot increase the supremum. 
			Now, observe that
			\begin{align}
				\tilde{f}\p{\frac{1}{2}(\bx_1+\bx_2)} &= \sin\p{\frac{1}{2}\pi 
				d^3\norm{\bx_1+\bx_2}^2} \nonumber\\
				&= \sin\p{\frac{1}{2}\pi 
				d^3\p{\norm{\bx_1}^2+\norm{\bx_2}^2+2\inner{\bx_1,\bx_2}}} 
				\nonumber\\
				&= \sin\p{\pi d^3 + \pi d^3\inner{\bx_1,\bx_2}} \nonumber\\
				&= (-1)^d\sin\p{\pi d^3\inner{\bx_1,\bx_2}}. \label{eq:daniely_func}
			\end{align}
			Note that approximating a function $ f $ and approximating its additive inverse $ -f $ using a neural network is equivalent (simply invert the weights of the output neuron), thus we can w.l.o.g.\ ignore the $ (-1)^d $ term in the above. Plugging \eqref{eq:daniely_func} in \eqref{eq:supsum} we obtain
			\begin{equation}\label{eq:summing_network}
				\sup_{\bx_1,\bx_2\in 
				\mathbb{S}^{d-1}}\abs{N^{\prime}(\bx_1+\bx_2) - \sin\p{\pi 
				d^3\inner{\bx_1,\bx_2}}} < \frac{1}{50\exp(2)\pi^2}
			\end{equation}
			Finally, let $ N^{\prime\prime}:\reals^{2d}\to\reals $ be the 
			network obtained from $ N^{\prime} $ by duplicating its first layer 
			weights excluding biases, i.e.\ $ \bw_i\mapsto(\bw_i,\bw_i) $, thus 
			we have that $ N^{\prime\prime}((\bx_1,\bx_2)) = 
			N^{\prime}(\bx_1+\bx_2) $ for any $ \bx_1,\bx_2\in\mathbb{S}^{d-1} 
			$. Plugging this in \eqref{eq:summing_network} we obtain
			\begin{equation}\label{eq:daniely_contradiction}
				\sup_{\bx_1,\bx_2\in 
				\mathbb{S}^{d-1}}\abs{N^{\prime\prime}((\bx_1,\bx_2)) - 
				\sin\p{\pi d^3\inner{\bx_1,\bx_2}}} < \frac{1}{50\exp(2)\pi^2}.
			\end{equation}
			That is, \eqref{eq:daniely_contradiction} establishes the existence 
			of a width $ w(d,101\exp(2)\pi^3d^3) $, depth $2$ ReLU network having 
			weights bounded by $ 2^d $, which uniformly approximates $ 
			\sin\p{\pi d^3\inner{\bx_1,\bx_2}} $ on $ 
			\mathbb{S}^{d-1}\times\mathbb{S}^{d-1} $ (and in particular, 
			provides such expected accuracy with respect to the uniform 
			distribution on $ \mathbb{S}^{d-1}\times\mathbb{S}^{d-1} $). By 
			\citet[Example 2]{daniely2017depth}, this implies that $ 
			w(d,101\exp(2)\pi^3d^3) \ge 2^{\Omega(d\log d)} $, concluding the proof 
			of \thmref{thm:daniely_reduction}

		\subsection{Proof of \thmref{thm:unit_ball_reduction}}
			In this proof, we utilize the measure on $ \reals^d $ 
			used in \citet{eldan2016power} for their lower bound, whose density is given by the square of
			\[
				\varphi(\bx) = \p{\frac{R_d}{\norm{\bx}_2}}^{d/2}J_{d/2}(2\pi 
				R_d\norm{\bx}_2)~,
			\]
			where $ R_d=\frac{1}{\sqrt{\pi}}\p{\Gamma\p{\frac{d}{2}+1}}^{1/d} 
			$, and $ J_{\nu}(z) $ is the Bessel function of the first kind, of 
			order $ \nu $ (see reference above for further information about 
			these functions). 
			
			Suppose we have $ N $ as in the theorem statement. Define $ 
			h_r(z)=\mathbbm{1}\set{z\le r} $ and
			\begin{equation*}
				f_{r,\delta}(z) = \begin{cases}
					1 & z\in(-\infty,r) \\
					-\frac{1}{\delta}z+\frac{r}{\delta}+1  & z\in[r,r+\delta] \\
					0 & z\in(r +\delta,\infty)
				\end{cases},
			\end{equation*}
			for some parameter $r$. That is, $ 
			f_{r,\delta} $ is a $ \frac{1}{\delta} $-Lipschitz approximation of 
			the indicator $ z\mapsto\mathbbm{1}\set{z\le r} $. To establish 
			\thmref{thm:unit_ball_reduction}, we shall fix $C=\frac{\epsilon^2\sqrt{d}}{5.2} $ and consider the measure $\mu$ with density $\gamma'(\bx)=(\beta\alpha)^d\varphi^2(\beta\alpha\bx)$ used in \citet[Thm.~1]{safran2017depth}, for some $\beta\in[1,2]$ and where $ \alpha\ge1 $ is 
			the universal constant from \citet{eldan2016power}. We will show that
			\begin{equation}\label{eq:unit_ball_approx}
				\norm{f_{r,C}(\norm{\cdot}_2)-h_r(\norm{\cdot}_2)}_{L_2(\mu)}\le\epsilon/2,
			\end{equation}
			and that there exists a depth $2$ network $ N^{\prime} $ which is based on $ N $, having width $ 
			2\cdot w\p{d,c\epsilon^{-3}} $ for some universal $ c>0 $, and satisfying
			\begin{equation}\label{eq:network_approx}
				\norm{N^{\prime}-f_{r,C}(\norm{\cdot}_2)}_{L_2(\mu)}\le\epsilon/2.
			\end{equation}
			This would imply \thmref{thm:unit_ball_reduction}, since Equations 
			(\ref{eq:unit_ball_approx}) and (\ref{eq:network_approx}) yield
			\[
				\norm{N^{\prime}-h_r(\norm{\cdot}_2)}_{L_2(\mu)} \le 
				\norm{N^{\prime} - f_{r,C}(\norm{\cdot}_2)}_{L_2(\mu)} + 
				\norm{f_{r,C}(\norm{\cdot}_2) - h_r(\norm{\cdot}_2)}_{L_2(\mu)} 
				\le \frac{\epsilon}{2} + \frac{\epsilon}{2} = \epsilon,
			\]
			and by plugging $ \epsilon=c_0/d^2 $ for some universal constant $ 
			c_0>0 $, we have from \citet[Thm.~1, Eq.~(4)]{safran2017depth} that 
			for any such depth $2$ neural network approximation of a ball 
			indicator of radius $r=\sqrt{d}$ w.r.t.\ the measure $\mu$ with density $\gamma'$ satisfying
			\[
				\norm{N^{\prime}-h_r}_{L_2(\mu)} < \frac{c_0}{d^2},
			\]
			it must hold that the width of $ N $ satisfies $ w(d,c_2d^6)\ge 
			c_3\exp(c_4d) $ for any $ d>c_1 $, some $ c_2>0 $ and small enough 
			$ c_3,c_4>0 $.
			
			We begin by proving \eqref{eq:unit_ball_approx}. Following a 
			similar approach as in \citet[Lemma 7]{eldan2016power}. We have by 
			definition that 
			\[
				\norm{f_{r,C}(\norm{\cdot}_2)-h_r(\norm{\cdot}_2)}_{L_2(\mu)}^2=\int_{\reals^d}\p{f_{r,C}(\norm{\bx})
				 - h_r(\norm{\bx})}^2(\beta\alpha)^d\varphi^2(\beta\alpha\bx)d\bx.
			\]
			Changing to polar coordinates, where $ 
			A_d=\frac{d\pi^{d/2}}{\Gamma(\frac{d}{2}+1)} $ denotes the volume of the unit hypersphere in $ \reals^d $, the above equals
			\[
				\int_{0}^{\infty}A_dz^{d-1}(f_{r,C}(z)-h_r(z))^2(\beta\alpha)^d\varphi^2(\beta\alpha z)dz = 
				\int_{0}^{\infty}A_d\frac{R_d^d}{z}(f_{r,C}(z)-h_r(z))^2J_{d/2}^2(2\pi\beta\alpha R_dz)dz.
			\]
			Using the definition of $ R_d $, this equals
			\[
				\int_{0}^{\infty}\frac{d}{z}(f_{r,C}(z)-h_r(z))^2J_{d/2}^2(2\pi R_dz)dz \le \int_{\sqrt{d}}^{\sqrt{d}+C}\frac{d}{z}J_{d/2}^2(2\pi\beta\alpha R_dz)dz,
			\]
			where the inequality is due to the definitions of $ f_{r,C},h_r $, since both functions are identical on $ 
			\pco{0,\infty} $ except for the interval $ \pcc{\sqrt{d},\sqrt{d}+C} $, where they 
			deviate from each other by at most $ 1 $. Moreover, for such $ z $ in the integration interval we have that Lemma 14 from \citet{eldan2016power} applies; therefore, the 
			integral is upper bounded by
			\[
				\int_{\sqrt{d}}^{\sqrt{d}+C}\frac{d}{z}\cdot\frac{1.3}{z\sqrt{d}}dz \le \frac{1.3C}{\sqrt{d}} = \frac{\epsilon^2}{4},
			\]
			where \eqref{eq:unit_ball_approx} follows by taking the square root.
			
			Moving to \eqref{eq:network_approx}, suppose we have $ N $ as in 
			the theorem statement, approximating $ f $ to accuracy $ 
			\frac{\epsilon}{4(C^{-1}r+1)}= \Theta(\epsilon^3) $. Namely, we have a depth $2$ network of width $ w\p{d,c\epsilon^{-3}} $ such that
			\[
				\int_{\bx\in\reals^d}\p{N(\bx)-f(\bx)}^2\gamma(\bx)d\bx \le \frac{\epsilon^2}{16\p{C^{-1}r+1}^2},
			\]
			for some constant $ c>0 $ and some density $ \gamma $. Let $ A $ be some invertible matrix to be determined later, and consider the change of variables $ \by=A\bx \iff \bx = A^{-1}\by $, $ d\bx=\abs{\det\p{A^{-1}}}\cdot d\by $, which yields
			\begin{equation}\label{eq:gamma_measure}
				\int_{\by\in\reals^d}\p{N\p{A^{-1}\by}-f\p{A^{-1}\by}}^2\cdot\gamma(A^{-1}\by)\abs{\det\p{A^{-1}}}\cdot d\by \le \frac{\epsilon^2}{16\p{C^{-1}r+1}^2}.
			\end{equation}
			In particular, we may choose $ \gamma(\bz)=\abs{\det\p{A}}\cdot\gamma'(A\bz) $ (note that this indeed defines a measure as readily seen by the change of variables $ \bx = A\bz $, $ d\bx=\abs{\det\p{A}}d\bz $, yielding $ \int_{\bz}\gamma(\bz)d\bz=\int_{\bx}\gamma'(\bx)d\bx=1 $). Plugging the chosen $ \gamma $ in \eqref{eq:gamma_measure} we obtain
			\begin{equation}\label{eq:gamma_approx}
				\int_{\by\in\reals^d}\p{N\p{A^{-1}\by}-f\p{A^{-1}\by}}^2\cdot\gamma'(\by) d\by \le \frac{\epsilon^2}{16\p{C^{-1}r+1}^2}.
			\end{equation}
			Observing that for any $ A $, $ N(A^{-1}\by) $ expresses a linear transformation of the input which can be simulated by an appropriate modification of the weights in the hidden layer of $ N $, we choose $ A=(r+C)\cdot I_d $ and $ A=r\cdot I_d $, where $ I_d $ is the $ d\times d $ identity matrix, to obtain
			\begin{equation}\label{eq:cr_approx}
				\int_{\bx\in\reals^d} \p{N\p{\frac{\bx}{r+C}}-f\p{\frac{\bx}{r+C}}}^2\gamma'(\bx) d\bx \le \frac{\epsilon^2}{16\p{C^{-1}r+1}^2},
			\end{equation}
			and
			\begin{equation}\label{eq:r_approx}
				\int_{\bx\in\reals^d} \p{N\p{\frac{\bx}{r}}-f\p{\frac{\bx}{r}}}^2\gamma'(\bx) d\bx \le \frac{\epsilon^2}{16\p{C^{-1}r+1}^2}.
			\end{equation}
			Now, consider the network $ N^{\prime} $ given by
			\[
				N^{\prime}(\bx) = (C^{-1} r+1)N\p{\frac{\bx}{r+C}}-C^{-1} 
				rN\p{\frac{\bx}{r}}.
			\]
			Note that this is indeed a depth $2$ network of width $ 2\cdot w\p{d,c\epsilon^{-3}} $ as a linear combination of depth $2$ networks.
			We will show that this network approximates
			\[
				f_{r,C}(\norm{\bx}) = (C^{-1} 
				r+1)f\p{\frac{\bx}{r+C}}-C^{-1} rf\p{\frac{\bx}{r}}.
			\]
			Compute taking the square roots of Equations (\ref{eq:cr_approx}) and (\ref{eq:r_approx}) to obtain
			\begin{align*}
				& \norm{N^{\prime}-f_{r,C}(\norm{\cdot}_2)}_{L_2(\mu)} \\ =& 
				\norm{(C^{-1} r+1)N\p{\frac{\bx}{r+C}}-C^{-1} 
				rN\p{\frac{\bx}{r}} - (C^{-1} r+1)f\p{\frac{\bx}{r+C}} + 
				C^{-1} rf\p{\frac{\bx}{r}}}_{L_2(\mu)} \\
				=& \norm{(C^{-1} r+1)\p{N\p{\frac{\bx}{r+C}} - 
				f\p{\frac{\bx}{r+C}}} - C^{-1} r \p{N\p{\frac{\bx}{r}} - 
				f\p{\frac{\bx}{r}}}}_{L_2(\mu)} \\
				\le& (C^{-1} r+1)\norm{N\p{\frac{\bx}{r+C}} - 
				f\p{\frac{\bx}{r+C}}}_{L_2(\mu)} + C^{-1} r \norm{N\p{\frac{\bx}{r}} - 
				f\p{\frac{\bx}{r}}}_{L_2(\mu)} \\
				\le& (C^{-1} r+1)\frac{\epsilon}{4(C^{-1}r+1)} + C^{-1} r 
				\frac{\epsilon}{4(C^{-1}r+1)} \le \frac{\epsilon}{4} + 
				\frac{\epsilon}{4} = \frac{\epsilon}{2},
			\end{align*}
			implying \eqref{eq:network_approx}, and concluding the proof of \thmref{thm:unit_ball_reduction}.
	
	\subsection*{Acknowledgements}
	This research is supported in part by European Research Council (ERC) Grant 754705.
	
	\bibliographystyle{abbrvnat}
	\bibliography{citations}

	\appendix
	
	\section{Trading-Off $L,\epsilon$ and Radius of 
	Support}\label{app:reductions}
	
	In this appendix, we formally show that given an inapproximability result 
	for neural networks, using an $L$-Lipschitz function, w.r.t.\ to some 
	distribution with support of radius $r$ and accuracy $\epsilon$, it is easy 
	to get an inapproximability result even for $1$-Lipschitz functions, at the 
	cost of scaling either $\epsilon$ or $r$ polynomially in 
	$L$:
	
	\begin{theorem}
	Let $f$ be an $L$-Lipschitz function on $\reals^d$, and $\mu$ a measure 
	over 
	$\reals^d$ with support bounded in $\{\bx:\norm{\bx}\leq r\}$ for some 
	$r\leq \infty$. Suppose that
	\[
	\inf_{n\in\Ncal} 
	\E_{\bx\sim\mu}\left[\left(n(\bx)-f(\bx)\right)^2\right]~\geq~ \epsilon~,
	\]
	where $\Ncal$ is some class of functions closed under scaling (namely, if 
	$n\in\Ncal$, then $\bx\mapsto a\cdot n(b\bx)$ for any $a,b>0$ is also in 
	$\Ncal$). 
	\begin{enumerate}
	\item Define the 
	$1$-Lipschitz function $\tilde{f}(\bx):=\frac{1}{L}f(\bx)$. 
	Then it holds that
	\[
	\inf_{n\in\Ncal} 
	\E_{\bx\sim\mu}\left[\left(n(\bx)-\tilde{f}(\bx)\right)^2\right]~\geq~ 
	\frac{\epsilon}{L^2}~.
	\]
	\item Define the $1$-Lipschitz function 
	$\hat{f}(\bx):=f\left(\frac{1}{L}\bx\right)$, and 
	the measure $\hat{\mu}$ by $\hat{\mu}(A) := \mu\left(\frac{1}{L}A\right)$ 
	for any set $A$ in the $\sigma$-algebra of $\mu$ (where 
	$\frac{1}{L}A:=\{\frac{1}{L}\bx:\bx\in A\}$ and assuming this set is also 
	in 
	the 
	$\sigma$-algebra 
	of $\mu$). Then $\hat{\mu}$ has a support bounded in $\{\bx:\norm{\bx}\leq 
	rL\}$, and
	\[
	\inf_{n\in\Ncal} 
	\E_{\bx\sim\hat{\mu}}\left[\left(n(\bx)-\hat{f}(\bx)\right)^2\right]~\geq~ 
	\epsilon~.
	\]
	\end{enumerate}
	\end{theorem}
	
	\begin{proof}
	By the assumptions, we have $\inf_{n\in\Ncal} 
		\E_{\bx\sim\mu}\left[\left(\frac{1}{L}n(\bx)-\frac{1}{L}f(\bx)\right)^2\right]~\geq~
		 \frac{\epsilon}{L^2}$, so the first part follows from definition of 
		 $\tilde{f}$ and the fact that $\Ncal$ is closed under scaling. As to 
		 the second part, the assertion on the support of $\hat{\mu}$ is 
		 immediate, and we have
		 \[
		 \inf_{n\in\Ncal} 
		 	\E_{\bx\sim\hat{\mu}}\left[\left(n(\bx)-\hat{f}(\bx)\right)^2\right]
		 ~=~
		 \inf_{n\in\Ncal} 
		 	\E_{\bx\sim \mu}\left[\left(n(L\bx)-\hat{f}(L\bx)\right)^2\right]
		 ~=~
 		 \inf_{n\in\Ncal} 
 		 	\E_{\bx\sim \mu}\left[\left(n(L\bx)-f(\bx)\right)^2\right],
		 \] 
		 which is at least $\epsilon$ by our assumptions and the fact that 
		 $\Ncal$ is closed under scaling.
	\end{proof}

\end{document}